\documentclass[journal,twoside]{IEEEtran}

\ifCLASSINFOpdf
\else
\fi

\usepackage{amro-common}
\usepackage[inline]{enumitem}
\usepackage{booktabs}
\usepackage{multicol}

\definecolor{colTtilde}  {RGB}{ 0, 99, 0}
\definecolor{colEstMeas} {RGB}{ 255, 164, 0}
\definecolor{colEstBlock}{RGB}{ 255, 0, 0}
\definecolor{colMeasLC}  {RGB}{ 159, 31, 239}
\definecolor{colRawMeas} {RGB}{ 0, 0, 255}

\acrodef{IMU}{inertial measurement unit}
\acrodef{DVL-INS}{\acs{DVL}-aided \acs{INS}}
\acused{DVL-INS}
\acrodef{INS}{inertial navigation system}
\acrodef{AUV}{autonomous underwater vehicle}
\acrodef{DVL}{Doppler velocity log}
\acrodef{LBL}{long baseline}
\acrodef{USBL}{ultrashort baseline}
\acrodef{GPS}{global positioning system}
\acrodef{LC}{loop-closure}
\acrodef{DCM}{direction cosine matrix}
\acrodef{LTI}{linear time-invariant}
\acrodef{MMSE}{minimum mean squared error}
\acrodef{MAP}{maximum a posteriori}
\acrodef{SDP}{semidefinite program}
\acrodef{CSP}{constraint satisfaction problem}
\acrodef{LMI}{linear matrix inequality}
\acrodef{SLAM}{simultaneous localization and mapping}
\acrodef{BA}{bundle adjustment}
\acrodef{KF}{Kalman filter}
\acrodef{InEKF}{invariant extended Kalman filter}
\acrodef{ANEES}{average normalized estimation error squared}
\acrodef{CSRS-PPP}{Canadian spacial reference system precision point positioning}
\acrodef{GNSS}{global navigation satellite system}
\acrodef{MCT}{Monte-Carlo trial}
\acrodef{AHRS}{attitude heading and reference system}
\acrodef{CRD}{Collaborative Research and Development}
\acrodef{NSERC}{Natural Sciences and Engineering Research Council of Canada}
\acrodef{MEUSMA}{McGill Engineering Undergraduate Student Masters Award}

\renewcommand{\pdf}[1]{\ensuremath{p}\left(#1 \right)}

\newcommand{\update}[1]{\colour{black}{#1}}
\newcommand{\secondupdate}[1]{\colour{black}{#1}}

\newcommand{\linsysA}{\mbf{A}}
\newcommand{\linsysB}{\mbf{B}}
\newcommand{\linsysL}{\mbf{L}}
\newcommand{\linsysC}{\mbf{H}}
\newcommand{\linsysM}{\mbf{M}}
\newcommand{\Xsol}{\mbf{X}^{\star}}
\newcommand{\insvar}[1]{\tilde{#1}}

\NewDocumentCommand{\dispins}{ O{} O{} }{\insvar{\mbf{r}}_{#2}^{\ifthenelse{\equal{#1}{}}{}{\dispspace} #1 \ifthenelse{\equal{#1}{}}{}{\dispspace}}}
\NewDocumentCommand{\disprvins}{ O{} O{} }{\insvar{\mbfrv{r}}_{#2}^{\ifthenelse{\equal{#1}{}}{}{\dispspace} #1 \ifthenelse{\equal{#1}{}}{}{\dispspace}}}
\NewDocumentCommand{\dcmins}{ O{} }{\insvar{ \mbf{C}}_{#1} {\ifthenelse{\equal{#1}{}}{}{\dcmspace}}}
\NewDocumentCommand{\dcminstrans}{ O{} }{\insvar{ \mbf{C}}_{#1}^{\trans} {\ifthenelse{\equal{#1}{}}{}{\dcmspace}}}
\NewDocumentCommand{\dcmrvins}{ O{} }{\insvar{ \mbfrv{C}}_{#1} {\ifthenelse{\equal{#1}{}}{}{\dcmspace}}}
\NewDocumentCommand{\poseins}{ O{} O{} }{\insvar{ \mbf{T}}_{#2}^{\ifthenelse{\equal{#1}{}}{}{\posespace} #1 \ifthenelse{\equal{#1}{}}{}{\posespace}}}
\NewDocumentCommand{\poservins}{ O{} O{} }{\insvar{ \mbfrv{T}}_{#2}^{\ifthenelse{\equal{#1}{}}{}{\posespace} #1 \ifthenelse{\equal{#1}{}}{}{\posespace}}}
\newcommand{\Prins}{ \insvar{ \mbf{P}}^{\,\textrm{r}}}

\newcommand{\Pins}{ \insvar{ \mbf{P}}}
\newcommand{\postvar}[1]{\hat{#1}}

\newcommand{\Ppost}{ \postvar{ \mbf{P}}}

\newcommand{\uest}{\mbfhat{u}}
\newcommand{\yest}{\mbfhat{y}}
\newcommand{\Qest}{\mbfhat{Q}}
\newcommand{\Cest}{\hat{\linsysC}}
\newcommand{\Rest}{\mbfhat{R}}

\newlength{\subfigheight}
\newsavebox{\subfigbox}

\Crefformat{figure}{#2Fig.~#1#3}
\Crefmultiformat{figure}{Fig.~#2#1#3}{ and~#2#1#3}{, #2#1#3}{, and~#2#1#3}

\makeatletter
\AtBeginDocument{
  \check@mathfonts
}

\begin{document}
%
%
%
%
%
%
%
%
\def \myJournal {IEEE Journal of Oceanic Engineering}
\def \myDoi {10.1109/JOE.2023.3286854}
\def \myPaperSiteName {}
\def \myPaperSiteLink {}
\def \myYear {2023}
\def \myPaperCitation{A. Al-Baali, T. Hitchcox, and J. R. Forbes, ``Combining DVL-INS and Laser-Based Loop Closures in a Batch Estimation Framework for Underwater Positioning,'' \textit{IEEE Journal of Oceanic Engineering}, 2023.}


\begin{figure*}[t]

\thispagestyle{empty}
\begin{center}
\begin{minipage}{6in}
\centering
This paper has been accepted for publication in \emph{\myJournal}. 
\vspace{1em}

This is the author's version of an article that has, or will be, published in this journal or conference. Changes were, or will be, made to this version by the publisher prior to publication.
\vspace{2em}

\begin{tabular}{rl}
DOI: & \myDoi\\
\end{tabular}

\vspace{2em}
Please cite this paper as:

\myPaperCitation

\vspace{15cm}
\copyright \myYear \hspace{4pt}IEEE. Personal use of this material is permitted. Permission from IEEE must be obtained for all other uses, in any current or future media, including reprinting/republishing this material for advertising or promotional purposes, creating new collective works, for resale or redistribution to servers or lists, or reuse of any copyrighted component of this work in other works.

\end{minipage}
\end{center}
\end{figure*}
\newpage
\clearpage
\pagenumbering{arabic} 

\fontdimen16\textfont2=\fontdimen17\textfont2
\fontdimen13\textfont2=5pt

\title{Combining DVL-INS and Laser-based Loop Closures in a Batch Estimation Framework for Underwater Positioning}

%
%
%

\author{Amro~Al-Baali,
  \and
  Thomas~Hitchcox,~\IEEEmembership{Graduate~Student~Member,~IEEE,}
  \and
  and James~Richard~Forbes,~\IEEEmembership{Member,~IEEE}
  \thanks{Manuscript received 30 August 2021; revised 5 July 2022; accepted 12 June, 2023. This work was supported by the Natural Sciences and Engineering Research Council of Canada (NSERC) and Voyis Imaging Inc. through the Collaborative Research and Development (CRD) program. The work of Amro Al-Baali was supported by the McGill Engineering Undergraduate Student Masters Award (MEUSMA) program. The work of Thomas Hitchcox was supported by the McGill Engineering Doctoral Award (MEDA) program.  This paper was recommended for publication by Associate Editor E. Brekke upon evaluation of the reviewers' comments. \textit{(Corresponding author: Amro Al-Baali.)}
  
  The authors are with the Department of Mechanical Engineering, McGill University, Montreal, QC H3A~0C3, Canada
    (e-mail:~\href{mailto:amro.al-baali@mail.mcgill.ca}{amro.al-baali@mail.mcgill.ca};
    \href{mailto:thomas.hitchcoxli@mail.mcgill.ca}{thomas.hitchcox@mail.mcgill.ca};
    \href{mailto:james.richard.forbes@mcgill.ca}{james.richard.forbes@mcgill.ca}).
  }}

%
%

\markboth{IEEE Journal of Oceanic Engineering.  Preprint Version.  Accepted June, 2023}{Al-Baali et al.: Combining DVL-INS and Laser-based Loop Closures in a Batch Estimation Framework for Underwater Positioning}%
%



\maketitle

\acused{DVL}
\begin{abstract}
    Correcting gradual position drift is a challenge in long-term subsea navigation.
    Though highly accurate, modern \ac{INS} estimates will drift over time due to the accumulated effects of sensor noise and biases, even with acoustic aiding from a Doppler velocity log (\ac{DVL}).
    The raw sensor measurements and estimation algorithms used by the DVL-aided~INS are often proprietary, which restricts the fusion of additional sensors that could bound navigation drift over time. 
%
    In this letter, the raw sensor measurements and their respective covariances are estimated from the \acs{DVL}-aided~\acs{INS} output using semidefinite programming tools.  The estimated measurements are then augmented with laser-based loop-closure measurements in a batch state estimation framework to correct planar position errors.
    The heading uncertainty from the \ac{DVL}-aided~\ac{INS} is also considered in the estimation of the updated positions.  
    The pipeline is tested in simulation and on experimental field data.  The proposed methodology reduces the long-term navigation drift by more than 30~times compared to the DVL-aided INS estimate. 
\end{abstract}
\acresetall
\acused{DVL-INS}


%
\IEEEpeerreviewmaketitle

\section{Introduction}
%
%
%
%

\acused{AUV}
\IEEEPARstart{A}{utonomous} underwater vehicles (\acp{AUV}) are used for a variety of tasks, including subsea metrology, oceanographic surveys, and bathymetric data collection in marine and riverine environments~\cite{Paull_AUV_2014, Menna_Towards_2019}.
Accurate localization and navigation is essential to ensure the spatial accuracy of the data gathered for these applications.

Electromagnetic signals decay rapidly in water, largely prohibiting the use of \ac{GPS} as a globally correcting sensor for underwater navigation.
\acused{GPS}
Accurate underwater localization solutions rely on \ac{LBL} and \ac{USBL} acoustic sensors, which are expensive and time-consuming to set up~\cite{Paull_AUV_2014}.
An alternative option is to use a high-fidelity
inertial navigation system (INS) with acoustic aiding from a Doppler velocity log (DVL), referred to collectively as a {DVL-INS} system.
\acused{INS}
\acused{DVL}

DVL-INS systems provide accurate attitude and depth estimates.
For example, the Sonardyne SPRINT-Nav~500 provides a heading accuracy on the order of \ang{0.04} and a depth accuracy on the order of \SI{0.01}{\percent} full scale~\cite{Sonardyne_Datasheet_2021}.  High-calibre DVL-INS systems can achieve a drift rate as low as \SI{0.02}{\percent} of distance travelled~\cite{Sonardyne_Datasheet_2021}, however without external correction the ${(x,y)}$ position estimate will continue to drift without bound.  The emphasis of this letter is to improve the long-term accuracy of \ac{AUV} navigation by bounding and reducing displacement errors using laser-based \ac{LC} measurements.

Loop-closure measurements are relative measurements between poses at non-consecutive time steps, and may be used as statistical constraints in the batch estimation problem~\cite{Bailey_Simultaneous_2006,Dellaert_Factor_2017}.
Loop-closure measurements are computed by processing vehicle-to-feature measurements provided by cameras~\cite{Sibley_Vast_2010}, sonar~\cite{Fallon_Relocating_2013, Li_Pose_2018}, or optical scanners.  In this letter, the \ac{LC} measurements are computed by processing laser data collected using a Voyis Imaging Inc. Insight Pro underwater laser scanner.  The laser data is used to detect and match previously observed seabed features, and the resulting submaps are then used to compute the \ac{LC} measurements~\cite{Hitchcox_Comparing_2020, Hitchcox_Point_2020}.

Traditionally, \ac{LC} measurements are fused with raw measurements coming from the \ac{IMU} and \ac{DVL} in a filtering or batch state estimation framework~\cite{Barfoot_State_2017, Farrell_Aided_2008, Dellaert_Factor_2017}.
Batch sensor fusion problems may be represented by a pose graph such as the one in \Cref{fig:ipe/pose_graph_raw_meas_ytildedvl.eps}.
The kinematic and measurement models are used to construct factors in the pose graph and are a necessary part of the inference algorithm~\cite{Dellaert_Factor_2017}. Unfortunately, the raw measurements, sensor models, and navigation algorithms used within DVL-INS systems are proprietary and are inaccessible to the user. Additionally, the cross covariance terms between the correlated DVL-INS estimates at different time steps are missing.
Without this information, the pose graph will look like the one presented in \Cref{fig:ipe/pose_graph_ins_est.eps}, where the unary factors are the DVL-INS pose estimates and there are no factors between nodes except for the \ac{LC} factors.  As such, \ac{LC} corrections will not propagate throughout the pose graph, and their benefit will not be fully exploited.


The approach proposed in this letter is to estimate the interoceptive and exteroceptive measurements used in the DVL-INS and then fuse these estimated measurements with the \ac{LC} measurements in a batch estimation framework.
That is, the objective is to first convert the pose graph in \Cref{fig:ipe/pose_graph_ins_est.eps} to an equivalent pose graph in \Cref{fig:ipe/pose_graph_est_meas.eps}, and then solve the latter pose graph using a standard least-squares optimization.
Since the INS heading and displacement estimates are correlated with one another, the INS heading {uncertainty} (\ie, covariance) is \emph{considered} when estimating the interoceptive measurements.

\begin{figure}[tb]
  \centering
  \begin{subfigure}{\columnwidth}
    \includegraphics[width=\textwidth]{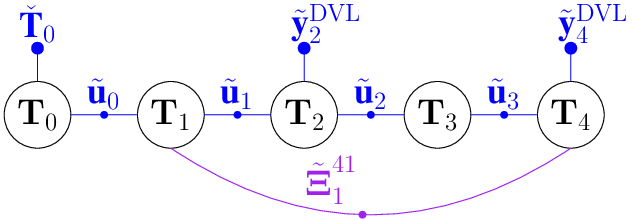}
    \caption{
    Traditional pose graph using raw measurements and \ac{LC} measurements.
    The prior \textcolor{colRawMeas}{$\mbfcheck{T}_{0}$}, the raw interoceptive measurements \textcolor{colRawMeas}{$\mbftilde{u}$}, and the raw \ac{DVL} measurements \textcolor{colRawMeas}{$\mbftilde{y}^{\mathrm{DVL}}$} are part of the DVL-INS and are \emph{not} available to the user.
    Therefore, such a pose graph is not realizable using the output of a commercial DVL-INS.}
    \label{fig:ipe/pose_graph_raw_meas_ytildedvl.eps}
  \end{subfigure}

  \begin{subfigure}{\columnwidth}
    \includegraphics[width=\textwidth]{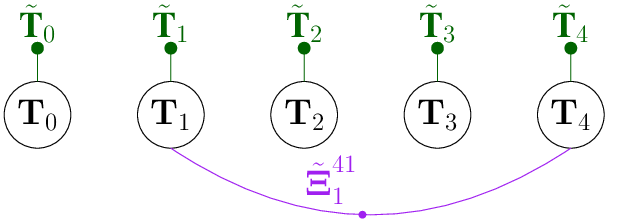}
    \caption{
      Pose graph using the pose estimates \textcolor{colTtilde}{$\mbftilde{T}$} from the DVL-INS and the \ac{LC} measurements \textcolor{colMeasLC}{$\mbs{\Xi}$}.
      The lack of factors between the variable nodes prevents the \ac{LC} corrections from propagating to other poses.
      That is, only poses $\mbf{T}_{1}$ and $\mbf{T}_{4}$ will be updated, while the other poses will still have the DVL-INS values \textcolor{colTtilde}{$\mbftilde{T}$}.
    }
    \label{fig:ipe/pose_graph_ins_est.eps}
  \end{subfigure}

  \begin{subfigure}{\columnwidth}
    \includegraphics[width=\textwidth]{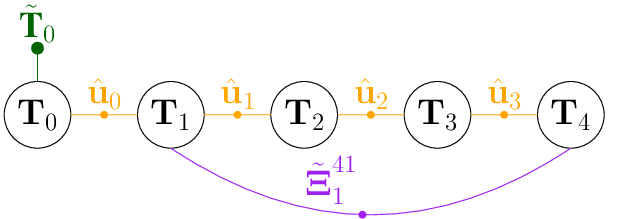}
    \caption{
      Pose graph using the estimated interoceptive measurements \textcolor{colEstMeas}{$\uest$} and the \ac{LC} measurements \textcolor{colMeasLC}{$\mbs{\Xi}$}.
      The problem is made observable by using the DVL-INS state estimate \textcolor{colTtilde}{$\mbftilde{T}_{0}$} as a prior, which is not necessarily the same as the prior \textcolor{colRawMeas}{$\mbfcheck{T}_{0}$} used in the DVL-INS or in the pose graph in \Cref{fig:ipe/pose_graph_raw_meas_ytildedvl.eps}.
      In this pose graph, the corrections from the \ac{LC} measurements are not limited to poses directly connected to the \ac{LC} corrections (\ie, $\mbf{T}_{1}$ and $\mbf{T}_{4}$).
    }
    \label{fig:ipe/pose_graph_est_meas.eps}
  \end{subfigure}
  \caption{
    Explaining the problem statement using pose graphs.
    The graph in \Cref{fig:ipe/pose_graph_raw_meas_ytildedvl.eps} is the traditional pose graph, while the pose graph in \Cref{fig:ipe/pose_graph_ins_est.eps} is the one available from the DVL-INS.  The proposed approach is to convert the given pose graph in \Cref{fig:ipe/pose_graph_ins_est.eps} to an equivalent pose graph \Cref{fig:ipe/pose_graph_est_meas.eps} by estimating the interoceptive measurements using the pose estimates.
  }
  \label{fig:pose graphs}
\end{figure}

The pipeline is presented as a flow chart in \Cref{fig:ipe/flow_chart_full.eps}.
Specifically, the DVL-INS enclosed in the blue dashed box is treated as a ``black-box'' in the sense that the user does not have access to the data, models, or algorithms used within the blue boxes.
The DVL-INS pose estimates $\mbftilde{T}$ are then used along with the laser scanner measurements $\mbftilde{y}^{\mathrm{laser}}$ to produce \ac{LC} measurements $\mbstilde{\Xi}^{\mathrm{loop}}$.
The \ac{LC} detection pipeline is presented in~\cite{Hitchcox_Comparing_2020, Hitchcox_Point_2020, Hitchcox2022a} and is not discussed further in this letter.
The pose estimates $\mbftilde{T}$ and the covariances from the DVL-INS are used to estimate equivalent interoceptive and exteroceptive measurements, denoted by $\uest$ and $\yest$, respectively.
This is done in the \emph{measurement estimation} block, coloured in red in \Cref{fig:ipe/flow_chart_full.eps}.
The estimated interoceptive measurements and the \ac{LC} measurements are then used to construct a factor graph similar to the one presented in \Cref{fig:ipe/pose_graph_est_meas.eps}, which in turn is solved using linear least-squares to produce the posterior displacement estimates $\disphat[zw][a]$.
Finally, the original DVL-INS attitude and depth estimates are combined with the updated displacement estimates to produce the posterior 3D pose estimate $\posehat$.
The work presented in this letter is specifically on the red blocks in \Cref{fig:ipe/flow_chart_full.eps}.

\begin{figure*}[htb]
  \centering
  \includegraphics[width=\textwidth]{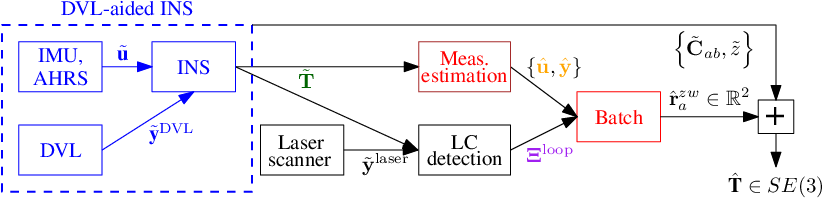}
  \caption{
    A flow chart of the pipeline.
    The \textcolor{blue}{blue} dashed box represents the \acs{DVL}-aided \acs{INS} that is treated as a black-box;
    the information within the \textcolor{blue}{blue} dashed box is not available to the user.  This includes body-centric linear velocity measurements from the DVL, linear acceleration and angular velocity measurements from the IMU, and a heading estimate from the attitude and heading reference system (AHRS).  The blocks in \textcolor{red}{red} are the main contributions of this work.
    \update{The \textcolor{red}{Meas. estimation} block is presented in \Crefrange{sec:Assumptions}{sec:Obtaining Process and Measurement Noise Covariances}, while the \textcolor{red}{Batch} block is presented in \Crefrange{sec:Considering the Heading Uncertainty}{sec:The Batch Estimation Problem}}.
    The \emph{measurement estimation} block estimates the interoceptive and exteroceptive measurements used in the INS from the INS pose estimates $\mbftilde{T}$;
    the \emph{batch} block is a linear least-squares optimization that uses the estimated measurements and the computed \ac{LC} measurement.
    The posterior displacement estimates $\disphat[zw][a]$ are combined with the INS attitude and depth estimates to form posterior 3D pose $\mbfhat{T}\in SE(3)$.
    The \emph{laser scanner} and \emph{\ac{LC} detection} blocks are presented in~\cite{Hitchcox_Comparing_2020, Hitchcox_Point_2020} and are not discussed in this letter.
  }
  \label{fig:ipe/flow_chart_full.eps}
\end{figure*}


The main challenges of the problem at hand are
\begin{enumerate*}[label=(\roman*)]
  \item the raw inertial measurements (\eg, \ac{DVL}, \ac{IMU}, and depth measurements) used in the {DVL-INS} are not available,
  \item the cross-covariance terms between poses at \emph{different} time steps are missing, and
  \item a lack of cross-covariance terms between attitude and displacement state estimates at the \emph{same} time step.
\end{enumerate*}

The novel contributions in this letter are:
\begin{itemize}
  \item estimating sensor raw measurements from post-processed state estimates (\eg, coming from a DVL-INS system) by posing a series of convex optimization problems;

  \item estimating the \update{white} noise process and measurement covariance matrices from incomplete posterior covariances using convex optimization tools while \emph{considering} the heading uncertainty.

\end{itemize}

\update{Note that the approach described in this letter is distinct from \cite{Hitchcox2022a}, in which a white-noise-on-acceleration motion prior is used to propagate loop-closure corrections throughout a DVL-INS trajectory estimate.  Specifically, the proposed approach uses semidefinite programming (SDP) techniques to estimate sensor measurements \textit{and covariance matrices} from the state estimate produced by a black-box DVL-INS system.  Additionally, trajectory corrections here are made on $\rnums^2$, whereas in \cite{Hitchcox2022a} trajectory corrections are made on $SE(3)$.}

The remainder of this letter is organized as follows.
Preliminaries are presented in \Cref{sec:Preliminaries}.
The methodology is presented \Cref{sec:Methodology}, which discusses the assumptions made, the formulation of the convex optimization problem to estimate the covariances, and the heading consider framework. The results of using the pipeline in simulation and on experimental data are presented in \Cref{sec:Simulations and Experiments}.  The paper concludes in \Cref{sec:Conclusion} with a summary and opportunities for future work.

\section{Preliminaries}
\label{sec:Preliminaries}

\subsection{Displacement and Attitude Notation}
\label{sec:Displacement and Attitude Notation}

A planar reference frame $\rframe[a]$ is composed of two orthonormal physical basis vectors.
The planar position of physical point $\update{z}$ relative to physical point $\update{w}$, resolved in reference frame $\rframe[\update{a}]$, is denoted as ${\update{\disp[zw][a]} \in \rnums^2}$.
The orientation of $\rframe[a]$ relative to $\rframe[b]$ is denoted here by a \ac{DCM} $\dcm[ab]$, where
${\dcm \in SO(2) = \left\{ \dcm \in \rnums^{2\times 2} \, | \, \dcm \dcmtrans = \eye, \det \dcm = +1 \right\}}$
\cite{Barfoot_State_2017, Sola_Micro_2020}.

In this letter, $\rframe[a]$ is used to describe the local tangent frame~\cite{Farrell_Aided_2008}, while $\rframe[b]$ is a reference frame that is fixed to and rotates with the vehicle.
Point $w$ is fixed in the world, while point $z$ is affixed to the vehicle.
The notation ${(\cdot)}_k$ is used to distinguish quantities at time $t_k$, such as $\mbf{C}_{ab_k}$ and $\mbf{r}^{z_{k}w}_a$.
\update{The reference frames are visualized in \Cref{fig:ipe/reference_frames.eps}.}

\begin{figure}[b]
  \centering
  \includegraphics[width=\columnwidth]{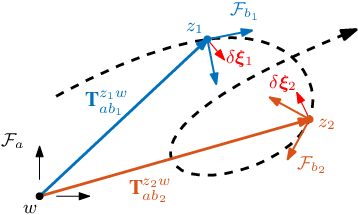}
  \caption{\update{Visualizing the physical points and reference frames involved in the subsea navigation problem.  The planar AUV pose at two instances in time, $t_1$ and $t_2$, is fully described by the transformations ${\mbf{T}^{z_1w}_{ab_1}, \mbf{T}^{z_2w}_{ab_2} \in SE(2)}$, respectively.  Local pose perturbations $\delta \mbs{\xi}_1$ and $\delta \mbs{\xi}_2$ are shown in red.}}
  \label{fig:ipe/reference_frames.eps}
\end{figure}

\subsection[Matrix Lie Group SE(2)]{Matrix Lie Group $SE(2)$}
\label{sec:liegroups}
The planar heading and position of a vehicle, collectively referred to as the vehicle `pose,' may be succinctly written as an element of matrix Lie group $SE(2)$~\cite{Barfoot_State_2017},
\begin{equation}
  \pose[z_{k}w][ab_{k}] =
  \bbm
  \dcm[ab_{k}]
  & \disp[z_{k}w][a]
  \\ \mbf{0}
  & 1
  \ebm
  \in SE(2),
\end{equation}
where
${SE(2) = \left\{ \pose \in \rnums^{3\times 3} \ | \ \dcm \in SO(2), \disp \in \rnums^2 \right\}}$.
Perturbations on $SE(2)$ are modelled in the Lie algebra of $SE(2)$, denoted $\mf{se}(2)$, which is defined as the tangent space at the group identity, $\mf{se}(2) \triangleq T_\eye SE(2)$~\cite{Barfoot_State_2017, Sola_Micro_2020}.
Here, perturbations take the form
\begin{equation}
  \pose = \posebar \exp(-\delta \mbs{\xi}^\wedge),
\end{equation}
where $\posebar$ is a nominal pose, $\exp(\cdot)$ is the matrix exponential, ${\delta \mbs{\xi} \in \rnums^3}$, and the operator ${(\cdot)^\wedge: \rnums^3 \to \mf{se}(2)}$ is an isomorphism between $\rnums^{3}$ and the Lie algebra given by~\cite{Sola_Micro_2020, Barfoot_State_2017}
\begin{equation}
  \delta \mbs{\xi}^\wedge
  =
  \bbm
  \delta \xi^\theta \\ \delta \xi^\textrm{r}_1 \\ \delta \xi^\textrm{r}_2
  \ebm^\wedge
  =
  \bbm
  0           & - \delta \xi^\theta    & \delta \xi^\textrm{r}_1    \\
  \delta \xi^\theta  & 0               & \delta \xi^\textrm{r}_2    \\
  0           & 0               & 0
  \ebm.
\end{equation}
%

\subsection{Random Variables}
\label{sec:randomvariables}

The notation $(\rv{\cdot})$ denotes a random variable.
Normally distributed variables are described by $\mbfrv{x} \sim \mc{N}\left(\mbf{x}, \mbf{P} \right)$, with mean $\mbf{x}$ and covariance $\mbf{P} = \cov{\rv{\mbfdel{x}}} = \mathbb{E} [ \rv{\mbfdel{x}} \, \rv{ \mbfdel{x}}^\trans ]$, where $\rv{\mbfdel{x}} = \mbfrv{x} - \mbf{x}$.
Covariance on $SE(2)$ poses is represented as $\mathbb{E} [ \rv{\mbsdel{\xi}} \, \rv{\mbsdel{\xi}}^\trans ]$.
Finally, the notation
$\hat{(\cdot)}$ is used to denote posterior estimates and $\insvar{(\cdot)}$ to denote input quantities such as data.




\subsection{Loop-Closure Measurements}
\label{sec:Loop Closure Measurements}
Loop-closure measurements are relative measurements between two poses computed by matching features from raw vehicle-to-feature measurements such as a laser finder~\cite{Dellaert_Factor_2017}, camera images~\cite{Sibley_Vast_2010}, ultrasound~\cite{Tardos_Robust_2002}, or optical scans~\cite{Hitchcox_Comparing_2020}.

Let $\mbf{T}_{1}, \mbf{T}_{2}\in SE(2)$ be the true poses at which the features are observed during the first and second passes, respectively.
The true \ac{LC} measurement of pose $\mbf{T}_{2}$ relative to pose $\mbf{T}_{1}$, resolved in the first pose frame, is given by
\begin{align}
  \mbs{\Xi}_{\ell}^{21}
                      & =
  \mbf{T}_{1}\inv \mbf{T}_{2}
  \\
                      & =
  \bbm
  \dcmtrans[1]        & -\dcmtrans[1]\disp[][1]
  \\ \mbf{0} & 1
  \ebm
  \bbm
  \dcm[2]             & \disp[][2]                            \\ \mbf{0} & 1
  \ebm
  \\
  \label{eq:LC meas. model expanded}
                      & =
  \bbm
  \dcmtrans[1]\dcm[2] & \dcmtrans[1](\disp[][2] - \disp[][1]) \\ \mbf{0} & 1
  \ebm,
\end{align}
where the subscript $\ell$ denotes the $\ell$-th \ac{LC} measurement.
The \emph{noisy} measurement is given by
\vspace{6pt}
\begin{align}
  \tilde{\mbsrv{\Xi}}_{\ell}^{21}
   & =
  \mbs{\Xi}_{\ell}^{21}\exp(-\rv{\mbsdel{\xi}}{}_\ell\expand),
  \vspace{6pt}
\end{align}
with
$
  {\rv{\mbsdel{\xi}}{}_\ell \sim \mc{N}\left(\mbf{0}, \mbf{R}_{\ell} \right)}
$
denoting the measurement noise.

\subsection{The Kalman Filter as a MAP estimator}
\label{sec:The Kalman Filter as a MAP estimator}

Consider a discrete-time \ac{LTI} system,
\begin{align}
  \label{eq:KF generic linear process model}
  \mbfrv{x}{}_k
   & = \linsysA\mbfrv{x}_{k-1} + \linsysB\mbf{u}_{k-1} + \linsysL\mbfrv{w}_{k-1},
\end{align}
where ${\mbf{x}_{k}\in\rnums^{n}}$ is the state, ${\linsysA\in\rnums^{n\times n}}$ is the transition matrix, ${\mbf{u}_{k-1} \in \rnums^m}$ is an interoceptive measurement, and ${\mbfrv{w}_{k-1}\sim\mc{N}\left(\mbf{0}, \mbf{Q}_{k-1} \right)}$ is the process \update{white} noise.
Furthermore, let the measurement model be
\begin{align}
  \label{eq:KF generic linear measurement model}
  \mbfrv{y}{}_{k} & = \linsysC_{k}\mbfrv{x}_{k} + \mbfrv{n}_{k},
\end{align}
where $\mbf{y}_k \in \rnums^p$ is an exteroceptive measurement, ${\mbf{H}_k \in \rnums^{p \times n}}$ is the measurement matrix, and ${\mbfrv{n}_{k}\sim\mc{N}\left(\mbf{0},\mbf{R}_{k} \right)}$ is the measurement noise.
The \ac{MMSE} estimator of $\mbfrv{x}_{k}$ given past measurements is the Kalman filter, which is also a \ac{MAP} estimator
\cite{ Barfoot_State_2017},
where the state estimate and the associated covariance are given by 
\vspace{6pt}
\begin{align}
  \mbfhat{P}_{k}\inv
   & =
  \left(\linsysA\mbfhat{P}_{k-1}\linsysA^{\trans} + \linsysL\mbf{Q}_{k-1}\linsysL^{\trans} \right)\inv + \linsysC_{k}^{\trans} \mbf{R}_{k}\inv\linsysC_{k},
  \label{eq:KF covariance eqn}
  \\
  \mbfhat{P}_{k}\inv\mbfhat{x}_{k}
   & =
  \left(\linsysA\mbfhat{P}_{k-1}\linsysA^{\trans} + \linsysL\mbf{Q}_{k-1}\linsysL^{\trans} \right)\inv\left(\linsysA\mbfhat{x}_{k-1} + \linsysB\mbftilde{u}_{k-1} \right)
  \nonumber                                                                                                    \\
   & \mathrel{\phantom{=}} \negmedspace {} + \linsysC_{k}^{\trans} \mbf{R}_{k}\inv\linsysC_{k}\mbftilde{y}_{k}
  \label{eq:KF state eqn}.
  \vspace{6pt}
\end{align}
\Cref{eq:KF covariance eqn,eq:KF state eqn} are known as the \update{information form} of the Kalman filter \update{\cite{Barfoot_State_2017}}.
A brief derivation of \eqref{eq:KF covariance eqn} and \eqref{eq:KF state eqn} is provided in Appendix~\ref{apx:Deriving the Kalman Filter Equations}.

\subsection{Batch Pose Estimation}
\label{sec:Batch Estimation}

Given exteroceptive measurements ${\mbftilde{y}_\ell, \ell = 1, \ldots, L}$, interoceptive measurements ${\mbftilde{u}_k, k=0, \ldots, K-1}$, and a prior estimate on the first pose, ${\mbfcheck{T}_0 = \mbfbar{T}_0 \exp(-\rv{\delta \mbscheck{\xi}}\vphantom{\mbf{\xi}}^\wedge_0)}$, ${\mbfcheck{P}_0 = \mathbb{E} [\rv{\delta \mbscheck{\xi}}^{}_0 \, \rv{\delta \mbscheck{\xi}}_0\vphantom{\mbf{\xi}}^\trans]}$, the \ac{MAP} solution to the batch pose estimation problem is given by~\cite{Barfoot_State_2017}
\vspace{6pt}
\begin{equation}
  \mbfhat{T} = \argmax_{\mbf{T}} p\left(
  \mbf{T} \, \big| \, \mbftilde{y}_{1:L}, \mbftilde{u}_{0:K-1}, \mbfcheck{T}_0
  \right).
  \label{eqn:batchjoint}
  \vspace{6pt}
\end{equation}
Under the Markov assumption,~\eqref{eqn:batchjoint} may be factored as
\begin{equation}
  \mbfhat{T} = \argmax_{\mbf{T}}
  \prod^L_{\ell=1} p \big(\mbftilde{y}_\ell \hspace{0.05em} \big| \hspace{0.05em} \mbf{T}_\ell \big) \prod^{K-1}_{k=0} p \big(\mbf{T}_{k+1} \hspace{0.05em} \big| \hspace{0.05em} \mbf{T}_k, \mbftilde{u}_k \big) p \big(\mbf{T}_0 \hspace{0.05em} \big| \hspace{0.05em} \mbfcheck{T}_0 \big).
  \label{eqn:batchfactored}
\end{equation}
Taking the negative log likelihood of~\eqref{eqn:batchfactored} produces a nonlinear least-squares problem,
\vspace{6pt}
\begin{equation}
  \mbfhat{T} = \argmin_{\mbf{T}} J(\mbf{T}),
  \label{eq:batchobj}
\end{equation}
\vfill\eject

where the objective function is
\vspace{6pt}
\begin{align*}
  J(\mbf{T}) = & \ \frac{1}{2} \sum^L_{\ell=1} \norm{ \mbf{e}_\ell (\mbftilde{y}_\ell, \mbf{g}_\ell(\mbf{T}_\ell))}^2_{\mbf{R}^{-1}_\ell} + \frac{1}{2} \norm{\mbf{e}_0(\mbfcheck{T}_0, \mbf{T}_0)}^2_{\mbfcheck{P}^{-1}_0} \\
               & \ + \frac{1}{2} \sum^{K-1}_{k=0} \norm{\mbf{e}_k(\mbf{f}_k(\mbf{T}_k, \mbftilde{u}_k), \mbf{T}_{k+1})}^2_{\mbf{Q}^{-1}_k},
  \numberthis
  \label{eq:batchsum}
  \vspace{6pt}
\end{align*}
where, in the general case, $\mbf{e}_\ell(\cdot)$, $\mbf{e}_k(\cdot)$, and $\mbf{e}_0(\cdot)$ denote the nonlinear measurement, process, and prior errors, respectively, and $\mbf{f}_k(\cdot)$ and $\mbf{g}_\ell(\cdot)$ represent the nonlinear process and measurement models, respectively.  The the notation ${\norm{\mbf{e}}^2_{\mbs{\Sigma}\inv} = \mbf{e}^\trans \mbs{\Sigma}\inv \mbf{e}}$ denotes the squared Mahalanobis distance.  Note that, for loop closure measurements, the measurement function involves more than one pose, ${\mbf{g}_\ell = \mbf{g}_\ell(\mbf{T}_1, \mbf{T}_2)}$. \Cref{eq:batchobj} is solved by iteratively relinearizing~\eqref{eq:batchsum} about the current state estimate, and minimizing the errors using, for example, Gauss-Newton or Levenberg–Marquardt~\cite{Barfoot_State_2017}.

\subsection{Semidefinite Programming}
\label{sec:Semidefinite Programmming}
\acused{SDP}
\update{
Semidefinite programming (\ac{SDP}) is a subfield of convex optimization, which has applications in control theory, covariance estimation, and more \cite{Caverly_LMI_2019, Boyd_Linear_1994, Boyd_Convex_2004}.
\Acp{SDP} gained popularity due to their expressiveness and strong theoretical and computational properties \cite{majumdarSurveyRecentScalability2019a}.
The theory and notation of \ac{SDP} is presented in this section and then used in \Cref{sec:Convexifying the Optimization Problem} to compute positive definite covariance matrices.
}

An \ac{SDP} problem has the form
\cite{Boyd_Convex_2004}
\begin{subequations}
  \label{eq:SDP standard form}
  \begin{alignat}{2}
     & \optmin_{
      \mbf{x} \in \rnums^{n}
    } \quad
     &                     &
    \mbf{c}^{\trans} \mbf{x}                       \\
     & \mathrm{s.t.} \quad
     &                     &
    \label{eq:LMI definition}
    \sum_{i=1}^{n} x_{i} \mbf{F}_{i} + \mbf{G} > 0 \\
     &
     &                     &
    \mbf{A} \mbf{x} = \mbf{b},
  \end{alignat}
\end{subequations}
where ${\mbf{G}, \mbf{F}_{i} \in \mbb{S}^{m}}$ for $i=1,\ldots,m$, and
\vspace{6pt}
\begin{align}
  \mbb{S}^{m}
  =
  \left\{
  \mbf{X} \in \rnums^{m \times m}
  \mid
  \mbf{X} = \mbf{X}^{\trans}
  \right\}
  \vspace{6pt}
\end{align}
is the set of $m \times m$ symmetric matrices.
The inequality \eqref{eq:LMI definition} is known as a \ac{LMI}, where it implies that the matrix on the left side is positive definite
\cite{Boyd_Convex_2004}.

The \ac{SDP} \eqref{eq:SDP standard form} may be written in matrix form as~\cite{Boyd_Convex_2004}
\vspace{6pt}
\begin{subequations}
  \begin{alignat}{2}
     & \optmin_{
      \mbf{X} \in \mbb{S}^{n}
    } \quad
     &           & \trace{ \left( \mbf{C} \mbf{X} \right)}                                                     \\
     & \st \quad
     &           & \trace{ \left( \mbf{A}_{i} \mbf{X} \right)} = b_{i}, \qquad i=1,\ldots, p                   \\
    \label{eq:LMI definition 2}
     &           &                                                            & \mbf{X} \geq 0,
  \end{alignat}
\end{subequations}
where $\trace(\cdot)$ is the trace operator and \eqref{eq:LMI definition 2} is a positive semidefiniteness constraint on $\mbf{X}$.
The notation ${\mbf{X}>0}$ and ${\mbf{X}\geq 0}$ implies that ${\mbf{X}\in\mbb{S}^{n}}$ is positive definite and positive semidefinite, respectively.
\ac{SDP} problems can be modelled using optimization-modelling toolboxes such as \textsc{yalmip}~\cite{Efberg_YALMIP}, which in turn solve the \ac{SDP} using fast and efficient \ac{SDP} solvers such as those available in \textsc{mosek}~\cite{ApS_MOSEK_2019}.
\vfill

\section{Methodology}
\label{sec:Methodology}

In this letter, it is assumed that the \ac{AUV} is equipped with a DVL-INS.
The DVL-INS produces the pose estimates
\vspace{6pt}
\begin{align}
  \poseins[z_{k}w][ab_{k}]
                  & =
  \bbm
  \dcmins[ab_{k}] & \dispins[z_{k}w][a] \\
  \mbf{0 }        & 1
  \ebm \in SE(3),
  \vspace{6pt}
\end{align}
the marginal covariances on the displacement $\Prins_{k}$, and the marginal covariance on the heading $\tilde{\sigma}_{\theta_k}^2$.
Neither the cross-covariance terms between poses ${\symExpect [\rv{\delta  \mbstilde{\xi}}{\vphantom{\mbs{\xi}}}_{k_1} \, \rv{\delta  \mbstilde{\xi}}{\vphantom{\mbs{\xi}}}^\trans_{k_2}], k_1 \ne k_2}$, nor the cross-covariance between the heading and displacement components $\symExpect [\rv{\delta \tilde{\xi}}{\vphantom{\xi}}^\theta_k \, \big(\rv{\delta \mbstilde{\xi}}{\vphantom{\mbs{\xi}}}^\textrm{r}_k \big)\vphantom{\mbs{\xi}}^\trans ]$ are accessible.

\update{The \ac{AUV} is also equipped with an Insight~Pro underwater laser scanner developed by Voyis Imaging Inc., pictured in \Cref{fig:insightpro}.  This sensor uses laser triangulation to generate high-resolution profiles of the seafloor.  The profiles are registered to the estimated AUV trajectory to generate point-cloud \textit{submaps}, from which loop-closure measurements are computed via a two-part point-cloud alignment algorithm, the details of which may be found in~\cite{Hitchcox_Comparing_2020, Hitchcox_Point_2020, Hitchcox2022a}.}

In the absence of raw interoceptive measurements and cross-covariance between poses at different time steps, there are no probabilistic models to couple the poses together.
That is, there are no binary factors in the pose graph.
An example of such a pose graph is presented in \Cref{fig:ipe/pose_graph_ins_est.eps}.
Without the binary factors, the \ac{LC} corrections will only propagate to the poses they are directly connected to, but not to other poses.
This limits the effectiveness of the batch solution, where it is possible to propagate the \ac{LC} corrections to many poses.

\begin{figure}[b]
  \centering
  \includegraphics[width=\columnwidth]{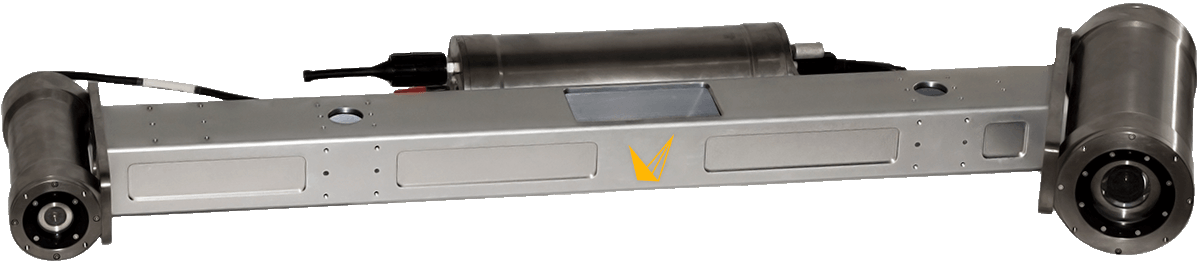}
  \caption{\update{An Insight Pro underwater laser scanner developed by Voyis Imaging Inc., which uses laser triangulation to measure high-resolution profiles of the seafloor.  The sensor baseline is approximately \SI{1}{\meter}.}}
  \label{fig:insightpro}
\end{figure}

The approach proposed in this paper is to \secondupdate{substitute a \textit{simplified process model} for the actual DVL-INS kinematics.  Following this simplification, semidefinite programming techniques are used to estimate the \textit{equivalent interoceptive measurements} that would produce the given DVL-INS trajectory estimate.  These equivalent measurements are then used in conjunction with the simplified process model to propagate newly-received loop-closure measurements within a batch optimization framework.} 
 That is, the pose graph in \Cref{fig:ipe/pose_graph_ins_est.eps} is to be converted to an approximately \emph{equivalent} pose graph as the one in \Cref{fig:ipe/pose_graph_est_meas.eps}, where the binary factors are constructed using the estimated interoceptive measurements.
The equivalent pose graph is then solved using linear least-squares.

\secondupdate{The use of loop-closure measurements to improve relative position estimates is not particularly novel within a conventional state estimation framework, in which raw measurements \textit{and their noise statistics} are known.  The contribution of this work is a method for incorporating loop-closure measurements into an existing trajectory estimate \textit{without access} to either raw sensor measurements or their underlying noise and bias characteristics.  What distinguishes the current approach from naive or existing approaches is \textit{the estimation of a covariance for each of the equivalent estimated interoceptive measurements}.  A covariance estimate ensures each interoceptive measurement is \textit{appropriately weighted} in the resulting linear least-squares problem in which loop-closure measurements are introduced.  This is the first approach to incorporating loop-closure measurements into a ``black-box'' DVL-INS trajectory estimate which considers an appropriate weighting for the estimated interoceptive measurements.}

The DVL-INS, which in this case is treated as a Kalman filter, is not invertible.
That is, there are infinitely many sets of measurements that, if passed through the Kalman filter, would result in the same set of state estimates.
Therefore, a set of assumptions is needed to formulate and solve an optimization problem for the measurements.

\update{
It should be noted that even though the computed quantities are referred to as \emph{retrieved} or \emph{estimated} measurements for succinctness, they are not true estimates of the underlying measurements.
That is, the estimated measurements obtained here may be far, in the Mahalanobis distance sense, from the raw sensor measurements generated by the DVL-INS.
These retrieved measurements can be thought of as \emph{some quantities}, that if used as measurements to estimate the vehicle state using a Kalman filter, would yield the same state estimate produced by the DVL-INS.
}

\subsection{Assumptions}
\label{sec:Assumptions}
The DVL-INS system will be treated as a Kalman filter,
where the discrete-time process model
\begin{align}
  \label{eq:assumed process model}
  \disprv[z_{k}w][a]
   & =
  \disprv[z_{k-1}w][a] + T_{k-1} \dcmins[ab_{k-1}]\uest_{k-1} + \mbf{L}\mbfhatrv{w}_{k-1}
\end{align}
is linear in the displacements,
${\dcmins[ab_{k-1}] \! \in SO(2)}$ is the \emph{known} heading estimate from the {DVL-INS},
$T_{k-1}$ is the sampling period,
$\uest_{k-1}$ are the interoceptive measurements to be estimated,
${\mbfhatrv{w}_{k-1}\sim\mc{N} (\mbf{0}, \Qest_{k-1} )}$ is the process noise,
and $\Qest_{k-1}$ is the process noise covariance to be estimated.
Furthermore, \update{to ensure the existence of a solution,} it is assumed that the filter is equipped with an exteroceptive sensor where the measurement model is given by
\begin{subequations}
\begin{align}
  \label{eq:assumed measurement model}
  \mbfhatrv{y}_{k}
   & =
  \linsysC_{k} \disprv[z_{k}w][a] + \mbfrv{n}_{k}
  \\
   & =
  \mbfhat{H}_k \dispins[z_{k}w][a] + \mbfhatrv{n}_{k},
\end{align}
\end{subequations}
where $\mbfhat{H}_{k}$ is the measurement matrix to be estimated, ${\mbfhatrv{n}_{k}\sim \mc{N}(\mbf{0}, \Rest_{k})}$ is the measurement noise, and $\Rest_{k}$ is the measurement noise covariance to be estimated.  The estimated exteroceptive measurements,
and their respective covariances, are needed to ensure the covariance estimation problem is well-posed.


\subsection{Estimating the Measurement Covariances}
\label{sec:Estimating the Measurement Covariances}
In contrast to the Kalman filter equations \eqref{eq:KF covariance eqn} and \eqref{eq:KF state eqn} where the objective is to compute the covariance on the state estimates given the noise covariances, the objective of the present approach is to estimate the noise covariances using the covariance on the state estimates.  That is, given covariances $\Pins_{k}$ and $\Pins_{k-1}$ from the {DVL-INS}, the objective is to find measurement matrices $\linsysC_{k}$, and covariances
${\mbf{Q}_{k-1}>0}$ and ${\mbf{R}_{k}>0},$ such that
\begin{align}
  \label{eq:information constraint eqn C'RinvC}
  \Pins_{k}\inv
   & =
  \left(\linsysA\Pins_{k-1}\linsysA^{\trans} + \linsysL\mbf{Q}_{k-1}\linsysL^{\trans} \right)\inv + \linsysC_{k}^{\trans} \mbf{R}_{k}\inv\linsysC_{k}
\end{align}
from \eqref{eq:KF covariance eqn} holds, where $\Pins_{i}$ are the covariances from the {DVL-INS}, and the constant process model matrices ${\linsysA=\eye}$ and $\linsysL$ are obtained from \eqref{eq:assumed process model}.

The number of design variables is reduced by defining the variable
\begin{align}
  \label{eq:Omegak = C'Rkm1invC}
  \mbs{\Omega}_{k}
   & \triangleq
  \linsysC_{k}^{\trans} \mbf{R}_{k}\inv \linsysC_{k}
  \geq 0,
\end{align}
where the positive semidefiniteness constraint arises from the fact that $\mbf{R}_{k}$ is positive definite and $\linsysC_{k}$ is generally a wide matrix with an associated null space.  Inserting \eqref{eq:Omegak = C'Rkm1invC} into \eqref{eq:information constraint eqn C'RinvC} results in
\begin{align}
  \label{eq:information constraint eqn Omega}
  \Pins_{k}\inv
   & =
  \left(\linsysA\Pins_{k-1}\linsysA^{\trans} + \linsysL\mbf{Q}_{k-1}\linsysL^{\trans} \right)\inv + \mbs{\Omega}_{k},
\end{align}
where the two design variables are $\mbf{Q}_{k-1} > 0$ and $\mbs{\Omega}_{k}\geq 0$.
The last term in \eqref{eq:information constraint eqn C'RinvC} and \eqref{eq:information constraint eqn Omega} is necessary for the existence of a solution.
To see this, consider the following counterexample where the last term is ignored.
\begin{example}
  Consider the {DVL-INS} covariances ${\Pins_{k}, \Pins_{k-1}\in\mbb{S}^{n}}$, where ${\Pins_{k}<\Pins_{k-1}}$.
  This is possible when step $k$ is a correction step~\cite{Barfoot_State_2017}.
  Furthermore, let $\mbf{A}=\mbf{L}=\eye$.
  Then, the $\mbf{Q}_{k-1}$ that satisfies \eqref{eq:information constraint eqn C'RinvC} is
  \begin{align}
    \mbf{Q}_{k-1}
     & = \Pins_{k} - \Pins_{k-1} < 0,
  \end{align}
  which does not satisfy the positive definiteness constraint on $\mbf{Q}_{k-1}$.
  Therefore, the problem is infeasible.
\end{example}
The \ac{CSP} \eqref{eq:information constraint eqn Omega} has a solution but the solution is not unique;
the proof is provided in \Cref{thm:existence and nonuniqueness of CSP}.
Therefore, an optimization problem with a meaningful objective function must be designed.

\subsection{Formulating the Optimization Problem}
\label{sec:Formulating the Optimization Problem}
Given that the \ac{AUV} is not equipped with a globally correcting sensor such as an \ac{LBL} or \ac{USBL} transceiver, the {DVL-INS} displacement estimates will drift over time.
\update{
  Thus, the measurements from the measurement model \eqref{eq:assumed measurement model} should have a minimal effect.
  Ideally, this is achieved by setting the measurement information matrix $\mbf{R}_{k}\inv$ to zero, which results in $\mbs{\Omega}_{k}$ being zero.
  However, as discussed at the end of \Cref{sec:Estimating the Measurement Covariances}, to keep the problem well-posed and ensure the existence of a solution, the variable $\mbs{\Omega}_{k}$ is to be minimized, but could be nonzero.
}
As such, using \eqref{eq:information constraint eqn Omega}, the objective function becomes
\begin{align}
  J(\mbf{Q}_{k-1})
   & =
  \fnorm{ \mbs{\Omega}_{k}}^{2}
  \\
  \label{eq:nonconvex objective function}
   & =
  \fnorm{
    \Pins_{k}\inv -
    \left(\linsysA\Pins_{k-1}\linsysA^{\trans} + \linsysL\mbf{Q}_{k-1}\linsysL^{\trans} \right)\inv
  }^{2},
\end{align}
where $\fnorm{\cdot}$ is the Frobenius norm.
The objective function is not a function of $\mbs{\Omega}_{k}$, which reduces the number of the unknown variables.
It should be noted that if $J(\mbfhat{Q}_{k-1})=0$ is achieved, then this implies that $\mbshat{\Omega}_{k}=\mbf{0}$, which in turn implies that there is no exteroceptive correction in the {DVL-INS} at time $k$.  The optimization problem is then
\vspace{6pt}
\begin{subequations}
  \label{eq:nonconvex optim problem}
  \begin{alignat}{2}
    \label{eq:nonconvex optim problem: obj func}
    \optmin_{
      \mbf{Q}_{k-1} \in \mbb{S}^{n}
    } \quad
     &
     &   &
    J(\mbf{Q}_{k-1})
    \\
    \label{eq:nonconvex optim problem: Omegak geq 0}
    \st \quad
     &
     &   &
    \Pins_{k}\inv -
    \left(\linsysA\Pins_{k-1}\linsysA^{\trans} + \linsysL\mbf{Q}_{k-1}\linsysL^{\trans} \right)\inv \geq 0,
    \\
    \label{eq:nonconvex optim problem: Qkm1 > 0}
     &
     &   &
    \mbf{Q}_{k-1} > 0
  \end{alignat}
\end{subequations}
where \eqref{eq:nonconvex optim problem: obj func} is given in \eqref{eq:nonconvex objective function}, and \eqref{eq:nonconvex optim problem: Omegak geq 0} comes from combining \eqref{eq:information constraint eqn Omega} with the positive semidefiniteness constraint on $\mbs{\Omega}_{k}$.

\subsection{Convexifying the Optimization Problem}
\label{sec:Convexifying the Optimization Problem}
Positive semidefineteness constraints, in the form of \acp{LMI}, can be enforced using \ac{SDP}
\cite{Boyd_Convex_2004}. The optimization problem \eqref{eq:nonconvex optim problem} is not an \ac{SDP} because
\begin{enumerate*}[label=(\roman*)]
  \item
  the objective function \eqref{eq:nonconvex objective function} is not convex, and

  \item
  the matrix inequality \eqref{eq:nonconvex optim problem: Omegak geq 0} is not affine (\ie, not an \ac{LMI})
\end{enumerate*}
in the design variable $\mbf{Q}_{k-1}$.
In order to use convex optimization tools, a substitution of variables is made to convert the optimization problem into a valid \ac{SDP}~\cite{Boyd_Convex_2004}. Define
\vspace{3pt}
\begin{align}
  \label{eq:Xkm1 definition}
  \mbf{X}_{k-1}
   & \triangleq
  \left(
  \linsysA\Pins_{k-1}\linsysA^{\trans}
  + \linsysL\mbf{Q}_{k-1}\linsysL^{\trans}
  \right)\inv.
  \vspace{3pt}
\end{align}
Then the positive definiteness of $\Pins_{k-1}$ and $\mbf{Q}_{k-1}$ and the full row rank of $\mbf{A}$ and $\mbf{L}$ imply
\vspace{3pt}
\begin{align}
  \label{eq:Xkm1 > 0}
  \mbf{X}_{k-1}>0.
  \vspace{3pt}
\end{align}
Inserting \eqref{eq:Xkm1 definition} into the nonconvex objective function \eqref{eq:nonconvex objective function} results in the objective function
\vspace{3pt}
\begin{align}
  \label{eq:convexified objective function}
  J^\prime \left(\mbf{X}_{k-1} \right)
   & =
  \fnorm{
    \Pins_{k}\inv - \mbf{X}_{k-1}
  }^{2},
  \vspace{3pt}
\end{align}
which is convex in $\mbf{X}_{k-1}$.
Similarly, the nonlinear inequality constraint \eqref{eq:nonconvex optim problem: Omegak geq 0} becomes
\vspace{3pt}
\begin{align}
  \label{eq:Pkm1 - Xkm1 >= 0}
  \Pins_{k-1}\inv - \mbf{X}_{k-1}
   & \geq 0,
   \vspace{3pt}
\end{align}
which is an \ac{LMI} in $\mbf{X}_{k-1}$.
To convert the inequality on $\mbf{Q}_{k-1}$
to an inequality on $\mbf{X}_{k-1}$, \eqref{eq:nonconvex optim problem: Qkm1 > 0} is first replaced
with the necessary condition
\begin{align}
  \label{eq:LQLt > 0}
  \linsysL \mbf{Q}_{k-1} \linsysL^{\trans}
   & > 0,
\end{align}
where \eqref{eq:LQLt > 0} implies \eqref{eq:nonconvex optim problem: Qkm1 > 0} for $\mbf{L}$ with full row rank, but not vice-versa, unless $\linsysL$ is a nonsingular matrix. Using \eqref{eq:Xkm1 definition}, the inequality \eqref{eq:LQLt > 0} is replaced with
\vspace{3pt}
\begin{align}
  \label{eq: nonlinear inequality constraint on Xkm1}
  \mbf{X}_{k-1}\inv - \linsysA\Pins_{k-1}\linsysA^{\trans}
   & > 0,
\end{align}
which is a nonlinear inequality constraint on $\mbf{X}_{k-1}$.
The nonlinear inequality constraint \eqref{eq: nonlinear inequality constraint on Xkm1} is converted to an \ac{LMI},
\begin{align}
  \label{eq: convexified LMI on Xkm1}
  \mbf{X}_{k-1} - \left(\linsysA \Pins_{k-1} \linsysA^{\trans} \right)\inv
   & < 0,
\end{align}
using \Cref{lemma:invertible strict LMI}. Combining the objective function \eqref{eq:convexified objective function} with the \acp{LMI} \eqref{eq:Xkm1 > 0}, \eqref{eq:Pkm1 - Xkm1 >= 0}, and \eqref{eq: convexified LMI on Xkm1} results in the \ac{SDP}
\begin{subequations}
  \label{subeq:SDP Xkm1}
  \begin{alignat}{2}
    \label{eq:SDP Xkm1: convex objective function}
    \optmin_{
      \mbf{X}_{k-1}\in\mbb{S}^{n}
    }
    \qquad    &
    \fnorm{
      \Pins_{k}\inv - \mbf{X}_{k-1}
    }^{2},
    \\
    \label{eq:SDP Xkm1: convex constraint on Xkm1}
    \st\qquad &
    \mbf{X}_{k-1} - \left(\linsysA \Pins_{k-1} \linsysA^{\trans} \right)\inv
              &       & < 0,
    \\
    \label{eq:SDP Xkm1: Pk-Xkm1>=0}
              & \quad
    \mbf{X}_{k-1} - \Pins_{k}\inv
              &       & \leq 0,
    \\
    \label{eq:SDP Xkm1: Xkm1>0}
              & \quad
    \mbf{X}_{k-1}
              &       & > 0,
  \end{alignat}
\end{subequations}
which can be solved using convex optimization tools such as \textsc{yalmip}~\cite{Efberg_YALMIP} and \textsc{mosek}~\cite{ApS_MOSEK_2019}.

\subsection{Obtaining Process and Measurement Noise Covariances}
\label{sec:Obtaining Process and Measurement Noise Covariances}
Once the \ac{SDP} \eqref{subeq:SDP Xkm1} is solved, the optimal process noise covariance matrix $\Qest_{k-1}$ and the optimal measurement noise covariance $\Rest_{k}$ are retrieved.
Let  $\Xsol$ be the solution to the \ac{SDP} \eqref{subeq:SDP Xkm1}.
Using \eqref{eq:Xkm1 definition}, the optimal process noise covariance matrix $\Qest_{k-1}$ is computed by solving the \ac{CSP}
\begin{subequations}
  \label{eq:CSP problem LQkm1L'}
  \begin{align}
    \linsysL \mbf{Q}_{k-1} \linsysL^{\trans}
     & = \left(\Xsol_{k-1}\right)\inv - \linsysA \Pins_{k-1}\linsysA^{\trans},
    \\
    \mbf{Q}_{k-1}
     & > 0,
  \end{align}
\end{subequations}
for $\mbf{Q}_{k-1}$, which can be computed analytically if $\linsysL$ is nonsingular.
The estimated measurement noise covariance $\mbf{R}_{k}$ and measurement matrix $\mbf{H}_{k}$ are not needed in the batch estimation pipeline, but the method to retrieve them is provided here for completeness.
Using \eqref{eq:information constraint eqn Omega} and \eqref{eq:Xkm1 definition}, the optimal information matrix $\mbs{\Omega}_{k}$ is
\begin{align}
  \mbs{\Omega}_{k}^{\star}
   & =
  \Pins_{k}\inv - \Xsol_{k-1}.
\end{align}
The optimal \update{measurement} matrix $\Cest_{k}$ and the measurement noise covariance matrix $\Rest_{k}$ are computed using an eigendecomposition.
Specifically, the positive semidefinite information matrix $\mbs{\Omega}^\star_{k}$ is decomposed into
\begin{align}
  \label{eq:CRinvC' dominant eigenvalues}
  \mbs{\Omega}_{k}^{\star}
                            & =
  \Cest_{k}^{\trans} \Rest_{k}\inv \Cest_{k}
  \\
                            & =
  \bbm
  \mbf{V}_{k, 1}            & \mbf{V}_{k, 2}
  \ebm
  \bbm \mbs{\Lambda}_{k, 1} &                \\ & \mbf{0} \ebm
  \bbm
  \mbf{V}_{k, 1}^{\trans}                    \\ \mbf{V}_{k, 2}^{\trans}
  \ebm
  \\
                            & =
  \mbf{V}_{k, 1}
  \mbs{\Lambda}_{k, 1}
  \mbf{V}_{k, 1}^{\trans},
\end{align}
where $\left(\mbs{\Lambda}_{k, 1}, \mbf{V}_{k, 1} \right)$ are the eigenpairs with positive eigenvalues, if they exist.
The optimal measurement and noise covariance matrices are, respectively,
\begin{align}
  \label{eq:Cest from eigenvector}
  \Cest_{k}
  =
  \mbf{V}_{k, 1}^{\trans},
  \qquad
  \Rest_{k}
  =
  \mbs{\Lambda}_{k, 1}\inv.
\end{align}

\subsection{Considering the Heading Uncertainty}
\label{sec:Considering the Heading Uncertainty}
As assumed in \Cref{sec:Assumptions}, the heading estimates from the {DVL-INS} are known to sufficient accuracy.
While highly precise, the heading estimates still retain uncertainty, and it is possible to account for that uncertainty by treating the heading as a noisy parameter.
This is similar to the methodology behind the `consider' Kalman filter, also referred to as a Schmidt-Kalman filter
\cite{Crassidis_Optimal_2012, Simon_Optimal_2006}.  This methodology is referred to herein as the \emph{consider framework}.
The consider framework only affects \Crefrange{sec:Obtaining Process and Measurement Noise Covariances}{sec:The Batch Estimation Problem},
meaning the methods discussed in \Crefrange{sec:Estimating the Measurement Covariances}{sec:Convexifying the Optimization Problem} remain unchanged.

In the consider framework, the heading uncertainty is \emph{considered} when estimating the process noise covariance matrix $\mbf{Q}_{k-1}$.
This is done by treating the heading in the process model \eqref{eq:assumed process model} as a random variable with a known mean and covariance.
The resulting process model is then
\begin{align}
  \label{eq:SE2 process model with random heading}
  \disprvins[z_{k}w][a]
   & =
  \disprvins[z_{k-1}w][a] + T_{k-1} \dcmrvins[ab_{k-1}]\uest_{k-1} + \mbf{L}\mbfrv{w}_{k-1},
\end{align}
where
\begin{align}
  \label{eq:INS random heading}
  \dcmrvins[ab_{k-1}]
   & = \dcmins[ab_{k-1}]\exp \big(-\rv{\delta\theta}_{k-1}\crossop \big),
\end{align}
and ${\rv{\delta\theta}_{k-1}\sim\mc{N} (0, (\tilde{\sigma}^{\theta}_{k-1})^2 )}$ is the heading noise.
The mean heading $\dcmins[ab_{k-1}]$ and the heading standard deviation $\tilde{\sigma}_{k-1}^{\theta}$ are known from the {DVL-INS}.  Inserting \eqref{eq:INS random heading} into the stochastic process model \eqref{eq:SE2 process model with random heading} and perturbing around the mean estimate yields
\begin{IEEEeqnarray}{rCl}
  \disprvins[z_{k}w][a]
  & = &
  \disprvins[z_{k-1}w][a] + T_{k-1} \dcmrvins[ab_{k-1}]\uest_{k-1} + \mbf{L}\mbfrv{w}_{k-1}
  \nonumber
  \\
  & = &
  \disprvins[z_{k-1}w][a] + T_{k-1} \dcmins[ab_{k-1}]\exp\big(-\rv{\delta\tilde{\theta}}_{k-1} \crossop\big)\uest_{k-1} + \mbf{L}\mbfrv{w}_{k-1}
  \nonumber
  \\
  & \approx &
  \disprvins[z_{k-1}w][a] + T_{k-1} \dcmins[ab_{k-1}]\uest_{k-1}
  \nonumber\\
  && \quad \negmedspace {} -  T_{k-1} \dcmins[ab_{k-1}]\rv{\delta\tilde{\theta}}_{k-1}\crossop\uest_{k-1} + \mbf{L}\mbfrv{w}_{k-1}
  \nonumber
  \\
  & = &
  \disprvins[z_{k-1}w][a] + T_{k-1} \dcmins[ab_{k-1}]\uest_{k-1}
  \nonumber\\
  && \quad \negmedspace {} -  T_{k-1} \dcmins[ab_{k-1}] \mbs{\Gamma} \uest_{k-1}\rv{\delta\tilde{\theta}}_{k-1} + \mbf{L}\mbfrv{w}_{k-1}
  \nonumber
  \\
  & = &
  \disprvins[z_{k-1}w][a] + T_{k-1} \dcmins[ab_{k-1}]\uest_{k-1}
  \nonumber\\
  && \quad \negmedspace {}
  +
  \underbrace{
  \bbm
  -T_{k-1} \dcmins[ab_{k-1}] \mbs{\Gamma} \uest_{k-1} &
  \mbf{L}
  \ebm
  }_{\mbf{L}_{k-1}'}
  \underbrace{
    \bbm
    \rv{\delta\tilde{\theta}}_{k-1} \\
    \mbfrv{w}_{k-1}
    \ebm
  }_{\mbfrv{w}^\prime_{k-1}},
\end{IEEEeqnarray}
where ${\theta^\times = \theta \mbs{\Gamma}}$,
\begin{align}
  \mbs{\Gamma}
   & = \bbm 0 & -1 \\ 1 & 0 \ebm,
\end{align}
and ${\exp (-\rv{\delta\tilde{\theta}}_{k-1} \crossop ) \approx \eye - \rv{\delta\tilde{\theta}}_{k-1} \crossop}$.  The process noise $\mbfrv{w}_{k-1}$ has been augmented in order to \textit{consider} the heading noise $\rv{\delta\tilde{\theta}}_{k-1}$.
This forms an updated noise column matrix, ${\mbfrv{w}_{k-1}' \sim \mc{N} ( \mbf{0}, \mbf{Q}_{k-1}' )}$, where
\begin{align}
  \label{eq:consider framework Qtilde}
  \mbf{Q}_{k-1}'
       & =
  \bbm
  (\tilde{\sigma}^{\theta}_{k-1})^2
       & \mbf{q}_{\mbf{w}_{k-1},\theta_{k-1}}^{\trans}
  \\
  \update{\mbf{q}_{\mbf{w}_{k-1},\theta_{k-1}}} & \mbf{Q}_{k-1}
  \ebm,
\end{align}
\update{with ${\mbf{q}_{\mbf{w}_{k-1},\theta_{k-1}} \! \in\mathbb{R}^{2}}$ the cross-covariance between the process noise $\mbfrv{w}_{k-1}$ and the heading noise $\rv{\delta\tilde{\theta}}_{k-1}$.}
%
The \ac{CSP} problem \eqref{eq:CSP problem LQkm1L'} then becomes
\begin{subequations}
  \label{eq:CSP problem LQkm1L' tilde}
  \begin{align}
    \mbf{L}_{k-1}' \mbf{Q}_{k-1}' \mbf{L}_{k-1}'^{\trans}
     & = \left(\Xsol_{k-1}\right)\inv - \linsysA \Pins_{k-1}\linsysA^{\trans},
    \\
    \mbf{Q}_{k-1}'
     & > 0,
    \\
    \label{eq:consider CSP problem: heading uncertainty constraint}
    \left[ \mbf{Q}_{k-1}' \right]_{(1, 1)}
     & = (\tilde{\sigma}^{\theta}_{k-1})^2,
  \end{align}
\end{subequations}
in the new design variable $\mbf{Q}_{k-1}' \in \mbb{S}^{3}$.
\subsection{Estimating Interoceptive Measurements from DVL-INS Output}
\update{
To reiterate the point discussed at \Cref{sec:Methodology}, the quantities referred to as \emph{retrieved} or \emph{estimated} measurements are not true estimates of the underlying data, but instead can be thought of as \emph{some quantities}, that if used as measurements to estimate the vehicle state using a Kalman filter, would yield the same state estimate produced by the DVL-INS.
}

The interoceptive measurements are computed by solving \eqref{eq:assumed process model} for $\mbf{u}_{k-1}$ using the {DVL-INS} estimates.
Specifically,
\begin{align}
  \label{eq:retrieve interoceptive measurements}
  \uest_{k-1}
   & =
  \tfrac{1}{T_{k-1}}\dcminstrans[ab_{k-1}]
  (\dispins[z_{k}w][a] - \dispins[z_{k-1}w][a]).
\end{align}
The interoceptive measurements are smoothed by taking a weighted sum of the previous $N$ interoceptive measurements.
If desired, the exteroceptive measurements are estimated using
\begin{align}
  \label{eq:yest = Cest star xins}
  \yest_{k}
   & = \Cest_{k} \dispins[z_{k}w][a],
\end{align}
where $\Cest_{k}$ is computed from \eqref{eq:Cest from eigenvector}, and $\dispins[z_{k}w][a]$ is the DVL-INS displacement estimate.

\subsection{The Batch Estimation Problem}
\label{sec:The Batch Estimation Problem}
Assuming a known heading estimate $\mbftilde{C}_{ab}$ from the DVL-INS, the process and measurement errors $\mbf{e}_{k}(\cdot)$ and $\mbf{e}_{\ell}(\cdot)$ from \Cref{sec:Batch Estimation} become linear functions of displacement only.  To see this, first note that the process model $\mbf{f}_k$ becomes
\begin{align}
  \mbf{f}_{k}(\disp[z_{k}w][a], \mbfhat{u}_k)
   & = \disp[z_{k}w][a] + T_{k-1}\dcmins[ab_{k}]\uest_{k},
\end{align}
while the measurement model $\mbf{g}_\ell$ becomes
\begin{equation}
  \mbf{g}_\ell \big(\disp[z_{k_2}w][a], \disp[z_{k_{1}}w][a] \big) = \mbftilde{C}_{ab_{k_1}}^\trans \big(\mbf{r}^{z_{k_2}w}_a - \mbf{r}^{z_{k_1}w}_a \big).
\end{equation}
This leads to a process error of
\begin{align}
  \mbf{e}_{k}(\disp[z_{k+1}w][a], \disp[z_{k}w][a], \mbfhat{u}_k)
   & =
  \disp[z_{k+1}w][a] - \mbf{f}_{k}(\disp[z_{k}w][a],\uest_{k}),
  \label{eqn:linearprocesserror}
\end{align}
and a measurement error (resolved in $\rframe[a]$) of
\begin{align*}
  \mbf{e}_\ell \big(\mbf{r}^{z_{k_2}w}_a\hspace{-4pt}, \mbf{r}^{z_{k_1}w}_a\hspace{-4pt}, \mbstilde{\Xi}^{k_2k_1}_\ell \big) = & \ \mbftilde{C}_{ab_{k_1}} \hspace{-3pt} \big(\mbf{g} \big(\disp[z_{k_2}w][a]\hspace{-4pt},\disp[z_{k_{1}}w][a] \hspace{-1pt} \big) - \mbftilde{y}_\ell \big(\mbstilde{\Xi}^{k_2k_1}_\ell \big) \big) \\
  =                                                                                                                            & \ \mbf{r}^{z_{k_2}w}_a - \mbf{r}^{z_{k_1}w}_a - \mbftilde{C}_{ab_{k_1}} \mbftilde{y}_\ell \big(\mbstilde{\Xi}^{k_2k_1}_\ell \big),
  \numberthis
  \label{eqn:linearmeaserror}
\end{align*}
\vspace{-3pt}
where
\begin{equation}
  \label{eq:linear LC meas. model y = C'(r2-r1)}
  \mbftilde{y}_\ell \big(  \mbstilde{\Xi}^{k_2k_1}_\ell \big) =   \begin{bmatrix}
    \eye & \mbf{0} \end{bmatrix}   \mbstilde{\Xi}^{k_2k_1}_\ell \begin{bmatrix}
    \mbf{0} \\ 1
  \end{bmatrix} \in \rnums^2.
\end{equation}
Both the process error~\eqref{eqn:linearprocesserror} and the measurement error \eqref{eqn:linearmeaserror} are linear functions of displacement only.  Denoting ${\mbf{R}^\textrm{r}_\ell = \mathbb{E} [ \rv{\mbsdel{\xi}}{}^\textrm{r}_\ell \, (\rv{\mbsdel{\xi}}{}^\textrm{r}_\ell)^\trans}]$, the covariance on the measurement error is
\vspace{-3pt}
\begin{equation}
  \mbf{R}_{e_\ell}^\textrm{r}
  =
  \cov{\mbf{e}_\ell}
  =
  \mbftilde{C}_{ab_{k_1}}  \mbf{R}^\textrm{r}_\ell  \mbftilde{C}_{ab_{k_1}}^\trans.
\end{equation}
Note that the estimated exteroceptive measurements $\yest_k$ are excluded from the batch solution.  This is because the number of estimated exteroceptive measurements is much larger than the number of \ac{LC} measurements.
Thus, if these estimated exteroceptive measurements are used with the \ac{LC} measurements in the batch estimation, then the estimated exteroceptive measurements would dominate the solution.  Furthermore, given that the estimated exteroceptive measurements are computed from the {DVL-INS} solution, using these exteroceptive measurements would pull the batch posterior closer to the {DVL-INS} solution, which in turn reduces the effectiveness of the \ac{LC} corrections.  This is further explored and justified in \Cref{sec:simdata} using simulated data.

\begin{figure}[b]
  \centering
  \includegraphics[width=\columnwidth]{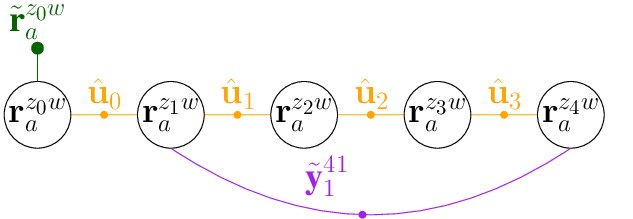}
  \caption{
    Factor graph of the linear least-squares problem \eqref{eq: linear lsq Jr-z} using the estimated interoceptive measurements \textcolor{colEstMeas}{$\uest$}, the {DVL-INS} displacement estimate \textcolor{colTtilde}{$\dispins[z_{0}w][a]$} as a prior, and the \ac{LC} measurements \textcolor{colMeasLC}{$\mbf{y}^{41}_{1}$} given by \eqref{eq:linear LC meas. model y = C'(r2-r1)}.
  }
  \label{fig:ipe/factor_graph_r_nodes.eps}
\end{figure}

\begin{figure*}[t]
  \sbox\subfigbox{%
    \resizebox{\dimexpr0.9\textwidth-1em}{!}{%
      \includegraphics[height=3cm]{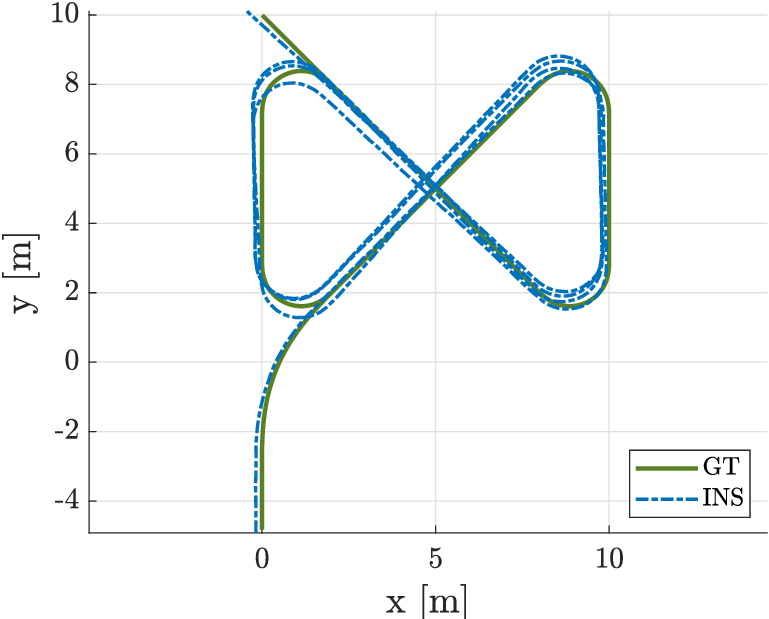}%
      \includegraphics[height=3cm]{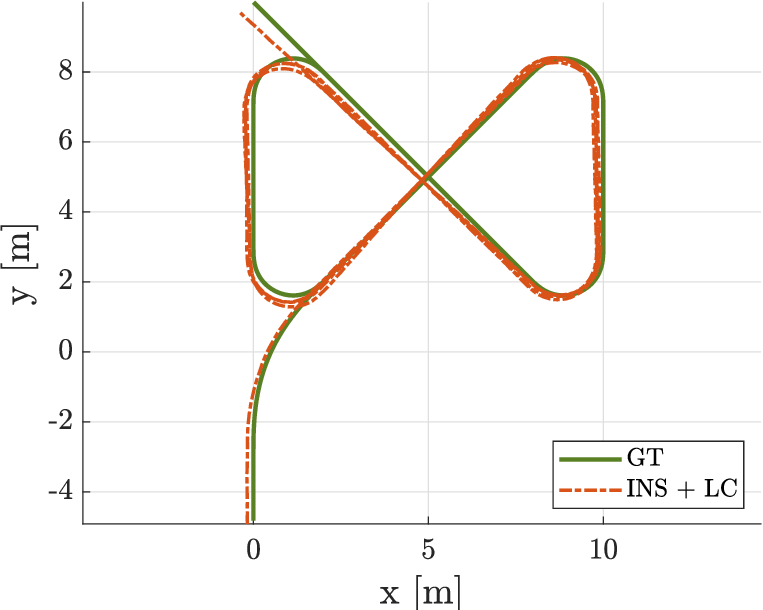}%
    }%
  }
  \setlength{\subfigheight}{\ht\subfigbox}
  \centering
  \subcaptionbox{
    A dead-reckoned estimate simulating \ac{INS} drift
    \label{fig:simulation/trajectories_without_inslc.eps}
  }{%
    \includegraphics[height=\subfigheight]{figs/simulation/trajectories_without_inslc.eps}
  }
  %
  \quad
  \subcaptionbox{
    Posterior estimate found via the methodology in \Cref{sec:Methodology}
    \label{fig:simulation/trajectories_with_inslc.eps}.
  }{%
    \includegraphics[height=\subfigheight]{figs/simulation/trajectories_with_inslc.eps}
  }
  \caption{
    Comparison between the simulated dead-reckoned trajectory (`INS') and the posterior trajectory computed using the methodology from \Cref{sec:Methodology} (`INS~+~LC').
    The vehicle passes over the area ${(x,y)=(5,5)}$ eight times, which results in seven \ac{LC} measurements relative to the first pass.  The `INS~+~LC' trajectory is more self-consistent compared to the drifting `INS' estimate, especially at the central intersection.
  }
  \label{fig:simulation trajectories}
\end{figure*}

The nonlinear batch problem of \Cref{sec:Batch Estimation} becomes
\vspace{-3pt}
\begin{align}
  \label{eq: linear lsq Jr-z}
  \disphat
   & =
  \argmin_{\disp}
  \frac{1}{2}
  \left(\mbf{J} \disp - \mbf{z}\right)
  \mbs{\Sigma}\inv
  \left(\mbf{J} \disp - \mbf{z}\right)^{\trans},
  \vspace{-3pt}
\end{align}
where
\begin{align}
  \disp
   & =  \bbm (\disp[z_{0}w][a])^{\trans}   & \cdots & (\disp[z_{K}w][a])^{\trans} \ebm^{\trans},
  \\
  \mbf{z}
   & = \left[
    \begin{array}{cccc}
      (\dispins[z_{0}w][a])^{\trans} & (T_{0}\dcmins[ab_{0}]\mbfhat{u}_{0})^{\trans} & \cdots & (T_{K-1}\dcmins[ab_{K-1}]\mbfhat{u}_{K-1})^{\trans}
    \end{array}
    \right.
    \nonumber                                                                                    \\
   & \mathrel{\phantom{=}} \negmedspace {}
    \left.
    \begin{array}{ccc}
      (\mbftilde{C}_{ab_{k_1}} \mbftilde{y}_1)^{\trans} & \cdots & (\mbftilde{C}_{ab_{k_1}} \mbftilde{y}_L)^{\trans}
    \end{array}
    \right]^{\trans},
  \\
  \mbf{J}
   & =
  \bbm
  \begin{array}{ccccc}
    \eye  &        &        &        &      \\
    -\eye & \eye   &        &        &      \\
          & -\eye  & \eye   &        &      \\
          &        & \ddots & \ddots &      \\
          &        &        & -\eye  & \eye \\
    \hline
          & -\eye  & \cdots & \eye   &      \\
          & \vdots &        &               \\
          & -\eye  & \cdots &        & \eye \\
  \end{array}
  \ebm,
  \\
  \mbs{\Sigma}
   & =
  \begin{bmatrix}
    \begin{array}{cccc|ccc}
      \mbftilde{P}_{0} &           &        &             &                          &        &                          \\
                       & \Qest_{0} &        &             &                          &        &                          \\
                       &           & \ddots &             &                          &        &                          \\
                       &           &        & \Qest_{K-1} &                          &        &                          \\
      \hline
                       &           &        &             & \mbf{R}^\textrm{r}_{e_1} &        &                          \\
                       &           &        &             &                          & \ddots &                          \\
                       &           &        &             &                          &        & \mbf{R}^\textrm{r}_{e_L}
    \end{array}
  \end{bmatrix}.
\end{align}
This least-squares optimization problem can be represented graphically using a factor graph~\cite{Dellaert_Factor_2017} similar to the one presented in \Cref{fig:ipe/factor_graph_r_nodes.eps}.  The analytic solution to \eqref{eq: linear lsq Jr-z} is~\cite{Barfoot_State_2017}
\begin{align}
  \disphat
   & = \left(\mbf{J}^{\trans}\mbs{\Sigma}\inv \mbf{J} \right)\inv \mbf{J}\mbs{\Sigma}\inv\mbf{z}.
\end{align}
The posterior estimates for the planar displacement ${\disphat[z_{k}w][a]\in\rnums^{2}}$ are augmented with the DVL-INS attitude estimates ${\dcmins[ab_{k}]\in SO(3)}$ and depth estimates $\insvar{z}_{k}$ to form the 3D pose,
\begin{align}
  \label{eq:fusing disp with attitude and depth}
  \posehat[z_{k}w][ab_{k}]
                  & =
  \bbm
  \dcmins[ab_{k}] & \bbm \disphat[z_{k}w][a] \\ \insvar{z}_{k} \ebm
  \\
  \mbf{0}         & 1
  \ebm.
\end{align}
This operation is denoted by the `$+$' box in the flow chart of \Cref{fig:ipe/flow_chart_full.eps}.
\update{
Finally, a pseudo-code summarizing the methodology presented in \Cref{sec:Methodology} is presented in \Cref{alg:INS+LC}.
}

\begin{algorithm}[t]
  \caption{\update{Summary of \Cref{sec:Methodology}.}}
  \label{alg:INS+LC}
  \begin{algorithmic}[1]
    \Statex \textbf{Input}: State estimates $\mbftilde{T}$, \ac{LC} measurements $\mbstilde{\Xi}^{\mathrm{loop}}$.%
    \thline%
    \For{each pose $\mbftilde{T}_{k}$ in $\mbftilde{T}$}
      \State Compute $\mbf{X}_{k-1}$ by solving the \ac{SDP} \eqref{subeq:SDP Xkm1}.
      \State \parbox[t]{\columnwidth-\leftmargin-\labelsep-\labelwidth}{Retrieve measurement covariance $\mbfhat{Q}_{k-1}$ by solving the \ac{CSP} \eqref{eq:CSP problem LQkm1L' tilde}.}\vspace{3pt}
      \State \parbox[t]{\columnwidth-\leftmargin-\labelsep-\labelwidth}{Estimate the interoceptive measurements $\mbfhat{u}_{k-1}$ \eqref{eq:retrieve interoceptive measurements}.}\vspace{3pt}
    \EndFor

    \State Estimate displacements $\disp[zw][a]$ by solving the linear least-squares problem \eqref{eq: linear lsq Jr-z} using the estimated interoceptive measurements $\mbfhat{u}$ and the \ac{LC} measurements $\mbstilde{\Xi}^{\mathrm{loop}}$.
    
    \State Augment the posterior displacement estimates $\disphat[zw][a]$ from the previous step with the attitude and depth estimates using \eqref{eq:fusing disp with attitude and depth}.
  \end{algorithmic}
\end{algorithm}


\section{Simulations and Experiments}
\label{sec:Simulations and Experiments}
The pipeline described in \Cref{sec:Methodology} is tested on simulated and experimental data.
Testing in simulation, where ground-truth information is available, allows for a comparison between batch solutions using either simulated or estimated sensor measurements.

\subsection{Performance Metric}
\label{sec:Performance Metric}
Due to the global unobservability of the \ac{SLAM} problem, the estimated and true trajectories may be misaligned, for example as shown in \Cref{fig:simulation trajectories}.
A location-invariant relative  metric, similar to the one proposed in~\cite{Kuemmerle_Measuring_2009}, is therefore used to assess the performance of the estimator.  The metric outlined in this section is for $SE(2)$ poses, but is also valid for $SE(3)$ poses.

Let $\pose[][k], \posehat[][k] \in SE(2)$ be the true and estimated poses, respectively, of the vehicle at time $t_{k}$.  The pose at time $t_{\ell}$ relative to the pose at time $t_{k}$ is
\begin{align}
  \label{eq:Kummerle error defintition}
  \delta\pose[{}][k\ell]
   & \triangleq
  \mbf{T}_{k}\inv \mbf{T}^{}_\ell,
  \vspace{-3pt}
\end{align}
where $t_{\ell}$ marks the earliest observation time of the first feature.  Furthermore, let
\vspace{-3pt}
\begin{align}
  \mbf{E}_{k}
                       & =
  \delta\mbfhat{T}_{k\ell}\inv \delta\mbf{T}^{}_{k\ell}
  \eqqcolon
  \bbm
  \delta\dcmhat[k\ell] & \delta\disphat[][k] \\
  \mbf{0}              & 1
  \ebm \in SE(2)
\end{align}
be the error between the estimated and true relative poses, where the estimated relative pose $\delta\mbfhat{T}_{k\ell}$ is computed from \eqref{eq:Kummerle error defintition} using estimated poses.
The metric used herein is the norm of the relative displacement error $\norm{\delta\disphat[][k]}$ at each time step $t_{k}$.
The relative displacement errors computed from a dead-reckoned estimate are expected to grow without bound, while an estimate that incorporates loop closures should produce bounded relative errors.



\subsection{Simulation}
\label{sec:simdata}

A simulation is set up to generate state estimates that resemble a drifting {DVL-INS}, as seen in \Cref{fig:simulation/trajectories_without_inslc.eps}.
\update{
Specifically, the true vehicle states, including the linear and angular velocities, are first computed from a given trajectory, and then the velocity measurements are corrupted with white noise to produce the noisy linear and angular interoceptive measurements $\mbf{u}$ and $\omega$, respectively.
These interoceptive measurements are then passed through the $SE(2)$ process model
\begin{align}
\label{eq:se2 process model}
\poseins[][k]
 &=
 \poseins[][k-1]
 \bbm
  T_{k-1}\omega_{k-1}\crossop & T_{k-1}\mbf{u}_{k-1} \\
  \mbf{0} & 1
 \ebm
\end{align}
to produce a set of state estimates $\poseins$, and the covariances are propagated by linearizing the process model \eqref{eq:se2 process model} with respect to the state and measurements.
}
The cross-covariance terms between the estimated heading and displacements are then ignored to mimic the DVL-INS output used in the open water experiments discussed in \Cref{sec:Experimental Data}.

Treating the dead-reckoned estimates as the {DVL-INS} estimates, the planar displacement is updated via the batch method given in \Cref{sec:The Batch Estimation Problem}. The updated state estimates are referred to as \emph{posterior} estimates and are denoted by $\hat{(\cdot)}$.
An example of a posterior trajectory computed using estimated measurements is presented in \Cref{fig:simulation/trajectories_with_inslc.eps}.

The metric discussed in \Cref{sec:Performance Metric} is used to compare the prior estimate (`INS') against the posterior estimate generated using \update{only} the estimated interoceptive measurements (`INS~+~LC')\update{, the posterior estimate generated using both the estimated interoceptive \emph{and exteroceptive} measurements (`INS~+~LC with ext. meas'), and} the posterior estimate generated using corrupted ground-truth measurements (`Odometry~+~LC~(batch)').
In all cases, the same \ac{LC} measurements are used.
A \ac{MCT} experiment is conducted over \update{10}~trials and the relative displacement error, averaged across trials, is shown in \Cref{fig:simulation/with_and_without_exteroceptive_meas/kummerle_energy.eps}.

\begin{figure}[tb]
  \centering
  \includegraphics[width=0.9\columnwidth]{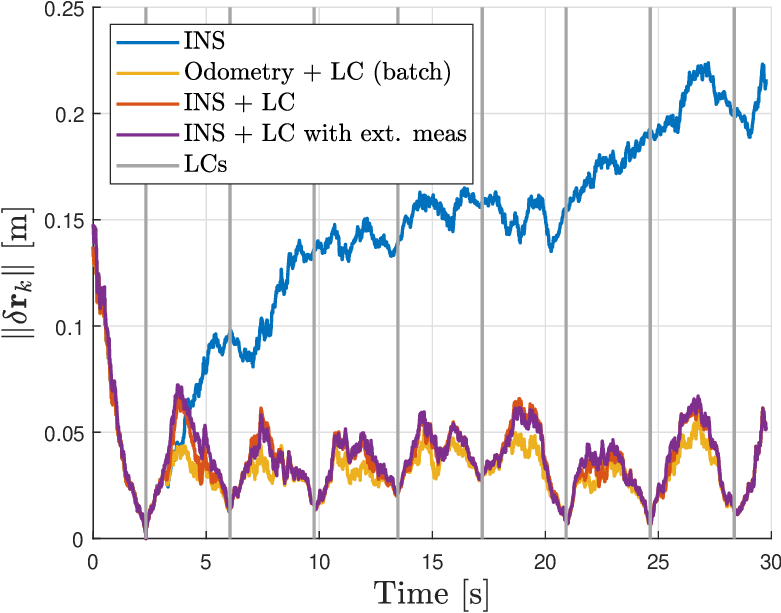}
  \caption{
  Results from simulated data showing the relative displacement error for different estimation solutions.
  The `INS' estimate is the dead-reckoned estimate, `{Odometry~+~LC~(batch)}' is the batch estimate using measurements generated from corrupted ground-truth data, \update{`{INS~+~LC}' and  `{INS~+~LC} with ext. meas' is the batch estimate computed via the methodology presented in \Cref{sec:Methodology}, where the former uses the exteroceptive measurements and the latter additionally uses exteroceptive measurements}.  `LCs' mark the timestamps at which the vehicle passes over the feature location.}
  \label{fig:simulation/with_and_without_exteroceptive_meas/kummerle_energy.eps}
\end{figure}

\begin{figure}[t]
  \centering
  \begin{subfigure}{0.45\textwidth}
    \includegraphics[width=\textwidth]{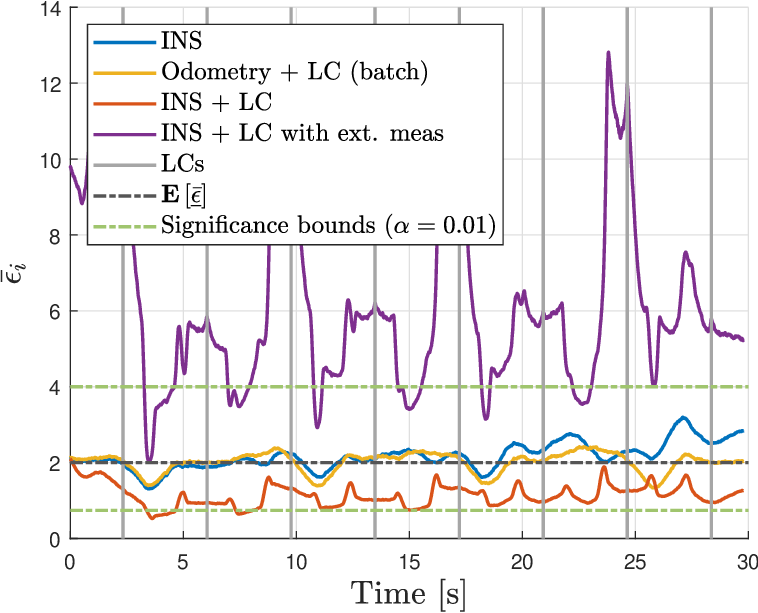}
    \caption{Average normalized estimation error squared (ANEES) test}
    \label{fig:simulation/with_and_without_exteroceptive_meas/mahalanobis.eps}
  \end{subfigure}
  \par\medskip
  \begin{subfigure}{0.45\textwidth}
    \includegraphics[width=\textwidth]{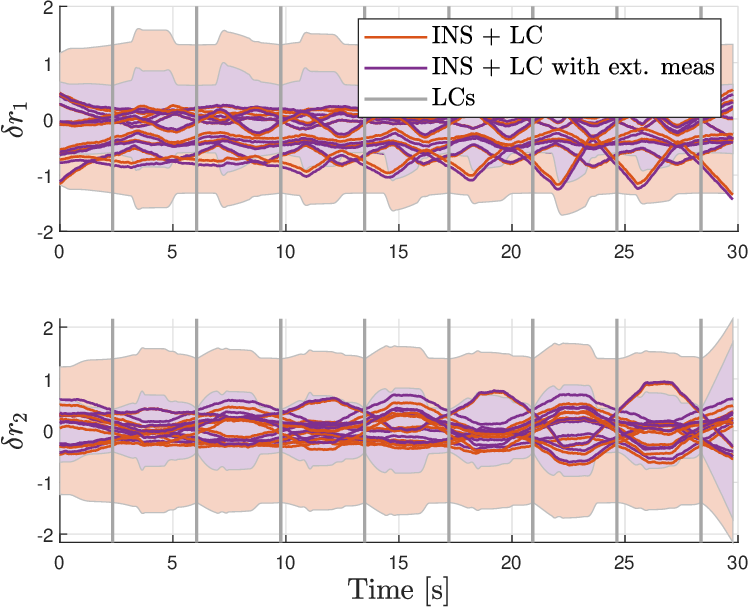}
    \caption{Mean displacement errors with $\pm3\sigma$ bounds}
    \label{fig:simulation/with_and_without_exteroceptive_meas/err_plt_blueOrange.eps}
  \end{subfigure}
  \caption{
    ANEES test and error plots on simulated data over \update{10}~Monte-Carlo trials.
    \update{The `INS' estimate is the dead-reckoned estimate, `{Odometry~+~LC~(batch)}' is the batch estimate using measurements generated from corrupted ground-truth data, and `{INS~+~LC}' and  `{INS~+~LC} with ext. meas' are the batch estimates computed via the methodology presented in \Cref{sec:Methodology}, where the latter uses the exteroceptive measurements and the former ignores them.  `LCs' mark the time stamps at which the vehicle passes over the same set of features.}
  }
  \label{fig:with_and_without_exteroceptive_meas/: ANEES and error plots on 10 MCT}
\end{figure}

The average relative error from the `INS~+~LC' estimates falls between the error from the dead-reckoned `INS' solutions and the error from the batch solutions using corrupted ground-truth measurements.  These results show that the average error associated with the `INS~+~LC' solution stays \update{relatively} bounded, as long as the vehicle passes over the first feature multiple times.

\begin{figure*}[t]
  \sbox\subfigbox{%
    \resizebox{\dimexpr0.9\textwidth-1em}{!}{%
      \includegraphics[height=3cm]{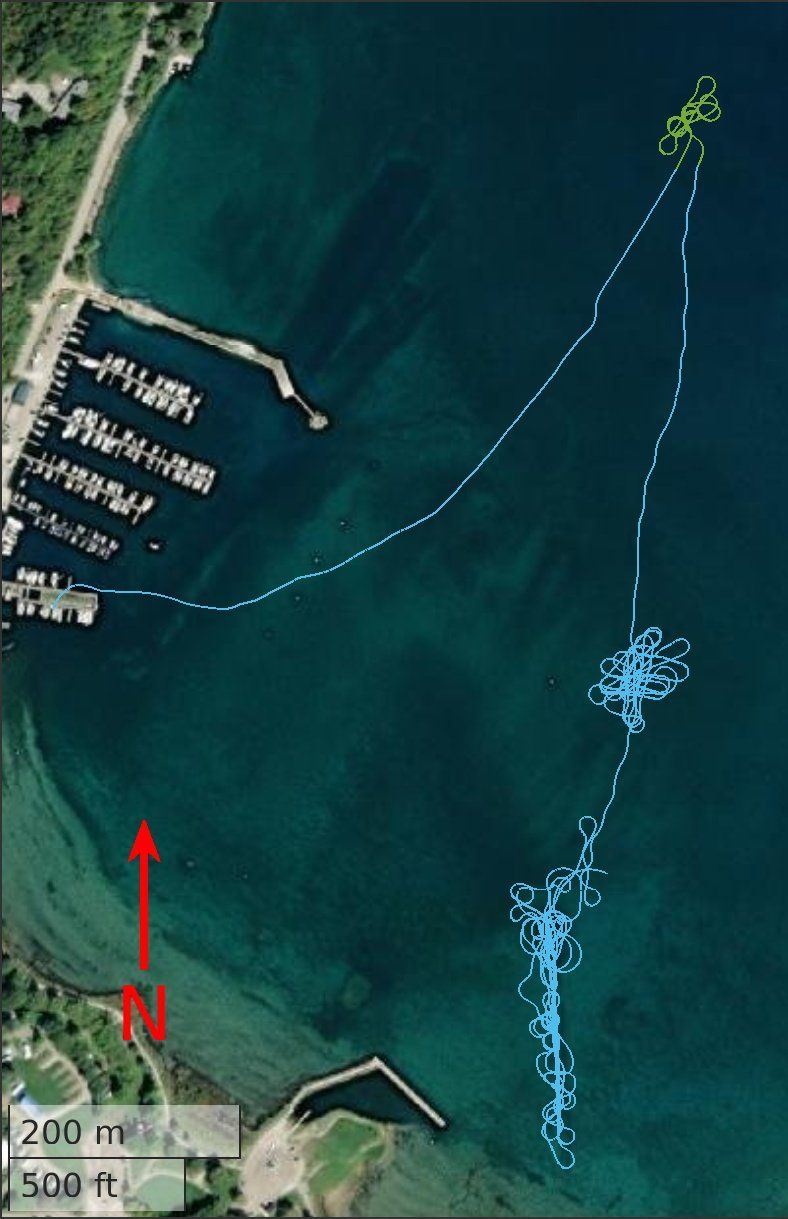}%
      \includegraphics[height=3cm]{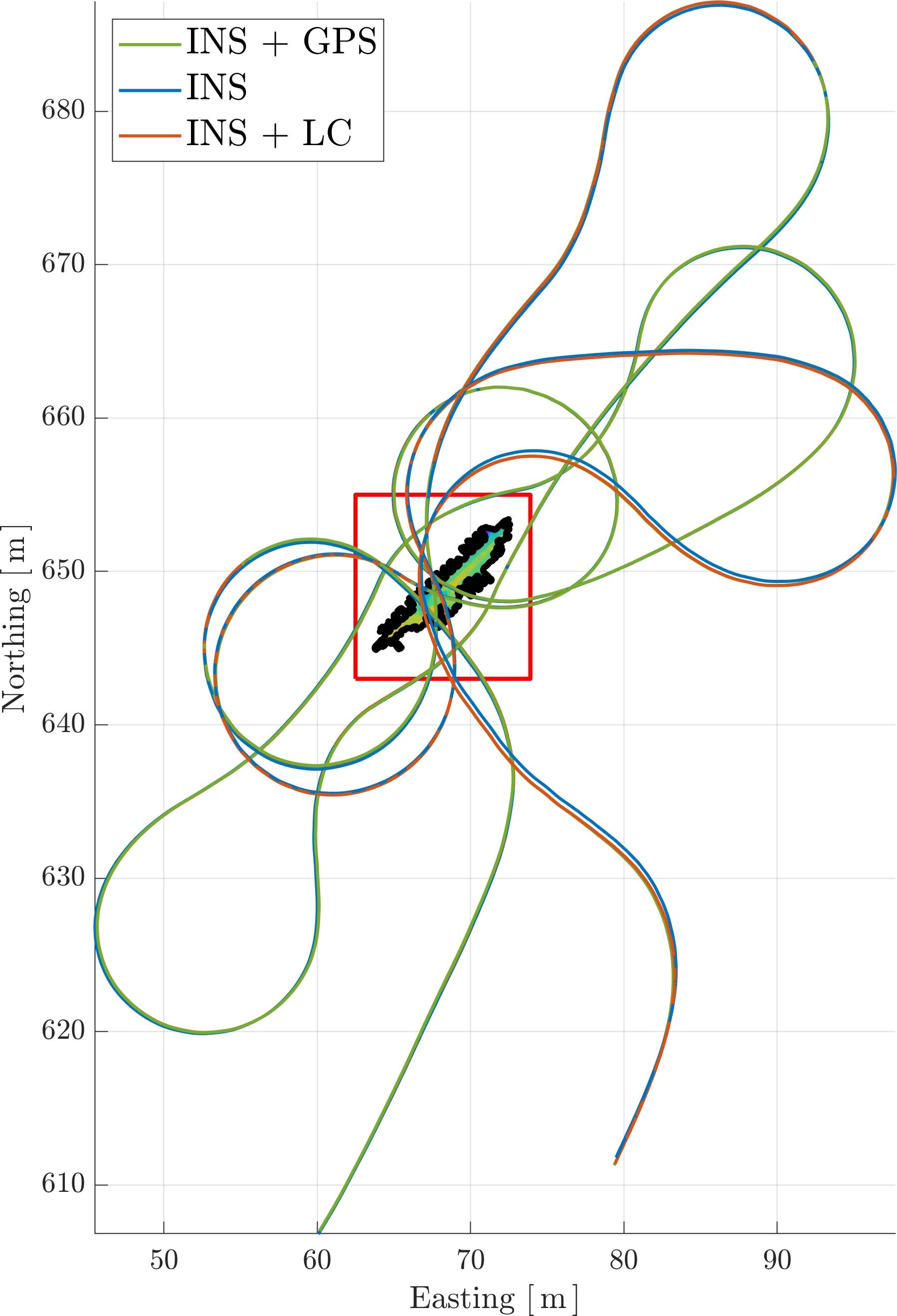}%
      \includegraphics[height=3cm]{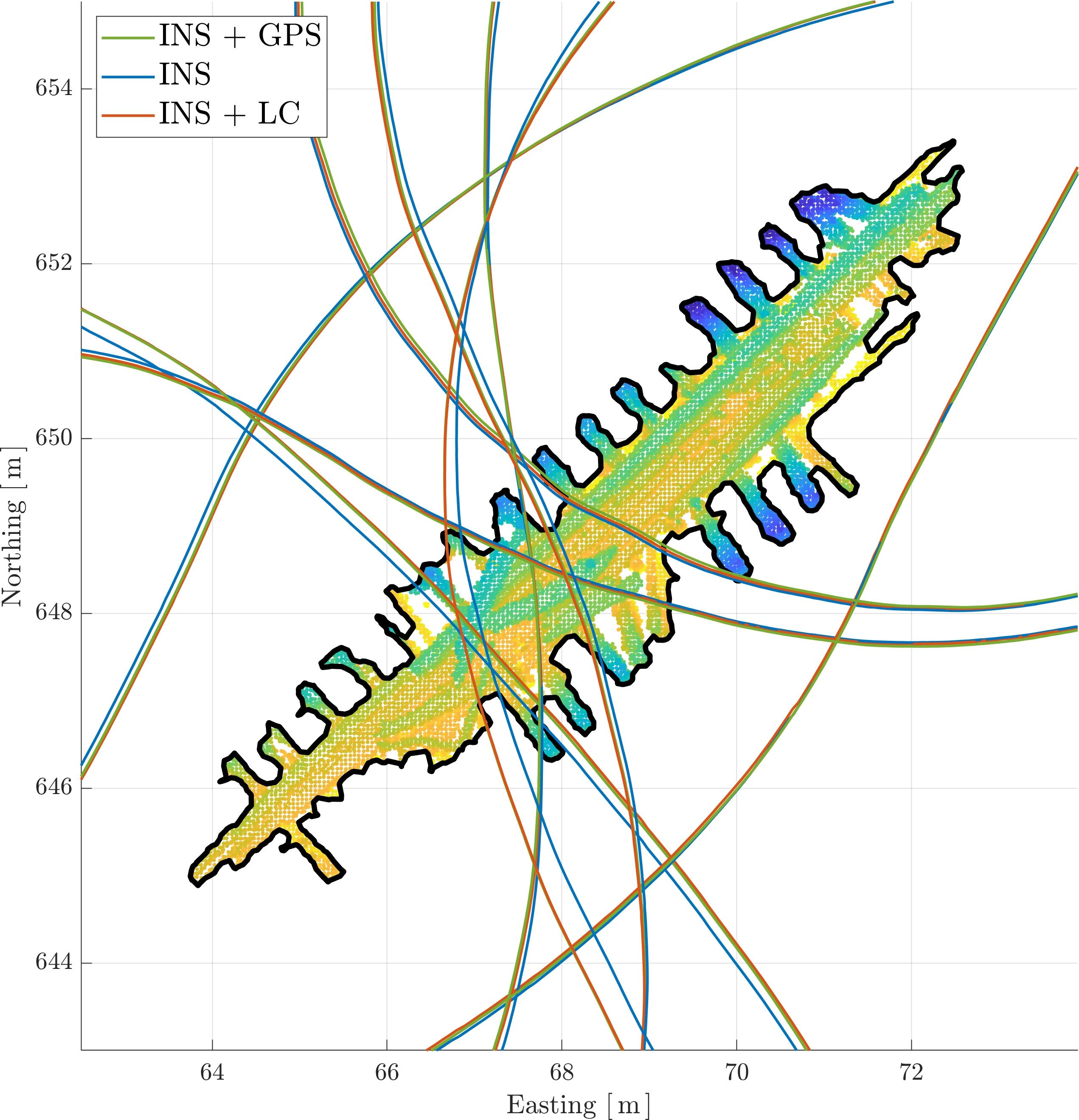}%
    }%
  }
  \setlength{\subfigheight}{\ht\subfigbox}
  \centering
  \subcaptionbox{The trajectory taken, with the region of interest in green.
    \label{fig:trajectory}}{%
    \includegraphics[height=\subfigheight]{figs/small/trajectory_map_cropped.jpg}
  }
  \quad
  \subcaptionbox{The region of interest, with different navigation solutions.
    The shipwreck area is highlighted. \label{fig:nav_solutions}}{%
    \includegraphics[height=\subfigheight]{figs/small/trajectories.jpg}
  }
  \quad
  \subcaptionbox{The zoomed region from \Cref{fig:nav_solutions}, showing the
    outline of the main shipwreck structure. \label{fig:nav_solutions_zoom}}{%
    \includegraphics[height=\subfigheight]{figs/small/trajectories_zoom_no_border.jpg}
  }
  \caption{
    Trajectory of experimental data collected in Colpoy's Bay, Wiarton, Ontario, Canada.
    The full trajectory in \Cref{fig:trajectory} is nearly \SI{7.5}{\kilo \meter} long, and
    took approximately \SI{2.4}{\hour} to collect.
    The section of interest is around \SI{0.58}{\kilo \meter} long, and took approximately \SI{10.5}{\minute} to collect.
    The trajectory makes eight passes over the main shipwreck structure.
  }
\end{figure*}

\renewcommand{\arraystretch}{1.25}
\begin{table*}[t]
  \centering
  \caption{Experimental navigation solutions}
  \label{tab:experimental navigation solutions}
  \begin{tabular}{@{}lp{0.8\textwidth}@{}}
    \toprule
    Solution
     & Description
    \\ \midrule
    INS~+~GPS
     &
    DVL-INS state estimates fused with \ac{GNSS} data from u-blox ZED-F9P high-precision \ac{GNSS} module.
    This estimate is used as ground truth when assessing the performance of the pipeline.
    \\
    INS
     & Dead-reckoned DVL-INS state estimates produced by a Sonardyne SPRINT-Nav~500~\cite{Sonardyne_Datasheet_2021}.
    Positioning precision was manually degraded by the industrial partner.
    \\
    INS~+~LC
     &
    Batch solution using estimated interoceptive measurements computed using the methodology in \Cref{sec:Methodology} and \ac{LC} measurements computed from Voyis Insight~Pro underwater scan data.
    \\
    \bottomrule
  \end{tabular}
\end{table*}

The overall consistency of a state estimate may be evaluated by computing the average normalized estimation error squared (ANEES) metric at each timestep \update{\cite[Sec.~5.4.3]{BarShalom_Estimation_2004}}.  \update{The ANEES metric follows a chi-square distribution with the degrees of freedom equal to the number of design variables.} \acused{ANEES}
The results of an ANEES test for all three solutions are presented in \Cref{fig:simulation/with_and_without_exteroceptive_meas/mahalanobis.eps}.  As the estimation problem involves two design variables (planar position), a value of ${\expect{\rv{\bar{\epsilon}}}=2}$ is expected.
Comparing the `INS', `Odometry~+~LC~(batch)', `INS~+~LC', and `INS~+~LC with ext. meas' solutions, the ANEES test shows that the proposed estimator is \update{overconfident when the exteroceptive measurements are included in the solution, but mildly underconfident when the exteroceptive measurements are ignored.
That is, the `INS~+~LC' estimated posterior covariances $\Ppost_{k-1}$ are smaller than the true covariances when the exteroceptive measurements are included, and larger than the true covariances otherwise.
This is also confirmed by the error plots in} \Cref{fig:simulation/with_and_without_exteroceptive_meas/err_plt_blueOrange.eps} showing the mean displacement errors and the $\pm3\sigma$ bounds.

\update{
As including the estimated exteroceptive measurements into the `INS~+~LC' solution produced inconsistent results, these measurements will be ignored.
However, ignoring the exteroceptive measurements results in a mildly underconfident estimator.  Tuning the confidence of the estimator will be addressed as part of future work.
}

\subsection{Experimental Data}
\label{sec:Experimental Data}

The full pipeline is tested on field data collected by industry partner Voyis~Imaging~Inc.  The data was collected in Colpoy's Bay, located in Wiarton, Ontario, Canada.  The full mission trajectory was nearly \SI{7.5}{\km} long and is shown in \Cref{fig:trajectory}.
A \SI{0.58}{\kilo \meter} section of the trajectory traversed a shipwreck area 8~times from which the laser data is used to compute \ac{LC} measurements.
A zoomed-in section of the trajectory is provided in \Cref{fig:nav_solutions,sub@fig:nav_solutions_zoom} along with the point-cloud scan of the shipwreck.

The sensor suite mounted on the surface vessel included a Sonardyne SPRINT-Nav~500 DVL-aided INS~\cite{Sonardyne_Datasheet_2021}, a Voyis~Insight~Pro underwater laser scanner, and a
u-blox {ZED-FP9} high-precision \ac{GNSS} module~\cite{Ublox_ZEDF9P02B_Datasheet_2021} equipped with a u-blox ANN-MB series high-precision multi-band antenna~\cite{Ublox_ANNMB_Datasheet_2021}. The positioning estimates from the \ac{GNSS} module were first processed using the Canadian Spacial Reference System
Precise Point Positioning (CSRS-PPP) application~\cite{Tetreault2005}, then fused with the DVL-INS to provide a high-precision position estimate of the vessel. These estimates are referred to as `INS~+~GPS' in this letter and are used as a ground-truth when assessing the performance of the pipeline.  The high-precision estimate was then reprocessed by the industry partner to remove the GNSS correction and inject additional position drift.  These estimates are referred to as `INS' estimates in this letter.  The laser measurements from the Voyis~Insight~Pro laser scanner are used along with the `INS' estimates to compute \ac{LC} measurements.  These \ac{LC} measurements are then used with the estimated raw measurements computed using the methodology presented in \Cref{sec:Methodology}, and the posterior estimates are referred to as `INS~+~LC.'  A summary of these three solutions is provided in \Cref{tab:experimental navigation solutions}.


\Cref{fig:wiarton/kummerle_energy.eps} shows the relative error metric from \Cref{sec:Performance Metric} computed for the `INS' and `INS~+~LC' solutions, where the `INS~+~GPS' solution is considered to be ground-truth.  The results show that the average error drops as the number of \acp{LC} increases. Therefore, theoretically, the error should remain bounded as long as there are recurring \acp{LC}.
The error for the dead-reckoned `INS' solution is \SI{0.1095}{\percent} of the total distance travelled (\SI{0.58}{\kilo \meter}).  After estimating sensor measurements and incorporating all loop closures (the `INS~+~7~LC' solution), this error drops to \SI{3.617e-3}{\percent} of total distance travelled, representing an improvement of more than \textbf{30 times}.  Furthermore, \Cref{fig:wiarton/kummerle_energy_LC7Only.eps} shows the effect of using a single \ac{LC} measurement.


A qualitative comparison is given by registering the laser profiles to the different trajectory estimates to produce point-cloud submaps.
\Cref{fig:Visualizing point clouds wiarton} shows the point clouds generated using the true (`INS~+~GPS'), prior (`INS'), and posterior (`INS~+~LC') trajectories. The point clouds in \Crefrange{fig:elevation_gt}{fig:elevation_posterior} are colour-coded by depth, whereas the clouds in \Crefrange{fig:disp_gt}{fig:disp_posterior} are colour-coded by geometric disparity~\cite{Roman_Consistency_2006}.

The `INS~+~LC' point~cloud in \Cref{fig:elevation_posterior} is more refined and `crisp' than the `INS' point~cloud in \Cref{fig:elevation_prior}.  The improvement is more visible in the zoomed-in images in the bottom row of \Cref{fig:shipwreck_passes}, where annotations highlight specific areas of the scan.  Furthermore, the posterior disparity in \Cref{fig:disp_posterior} has a higher concentration of blue points than the prior disparity in \Cref{fig:disp_prior}, indicating less severe disparity errors.  The red circles in these figures highlight two areas for which the posterior point~cloud has a lower disparity.

It should be noted that \Cref{fig:disp_prior} contains one green point-cloud section with high disparity values.  These errors are due to a bias in the {DVL-INS} depth estimate on one of the passes.
As shown in \Cref{fig:ipe/flow_chart_full.eps}, the proposed pipeline does not currently correct for errors in depth, and thus this bias remains in the posterior estimate.



\begin{figure}[tb]
  \centering
  \includegraphics[width=\columnwidth]{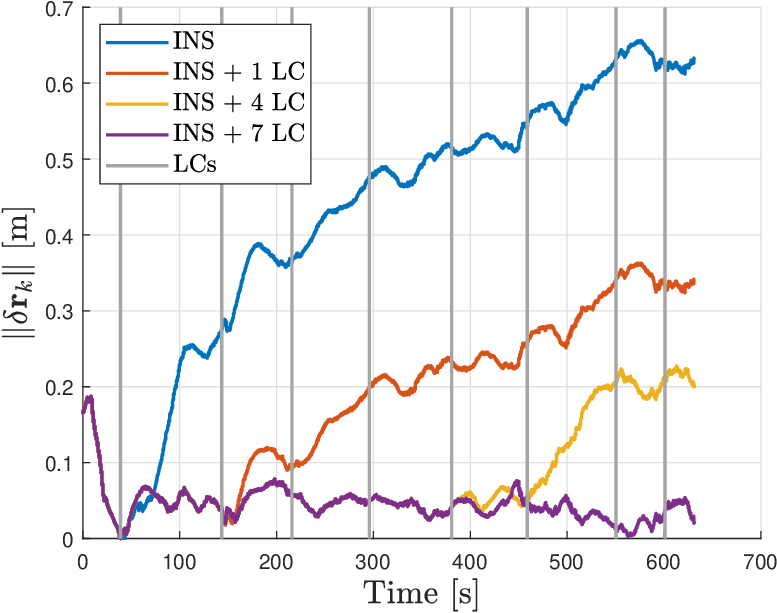}
  \caption{
    Relative displacement error $\| \delta \mbfhat{r}_k \|$ declines as loop closures are incorporated into the experimental data.  The shipwreck area was first observed around \SI{40}{\s}, with all LC measurements computed relative to the first observation.  Note that a Kalman filter run using the estimated measurements would overlap exactly with the `INS' estimate.
  }
  \label{fig:wiarton/kummerle_energy.eps}
\end{figure}
\begin{figure}[tb]
  \centering
  \includegraphics[width=\columnwidth]{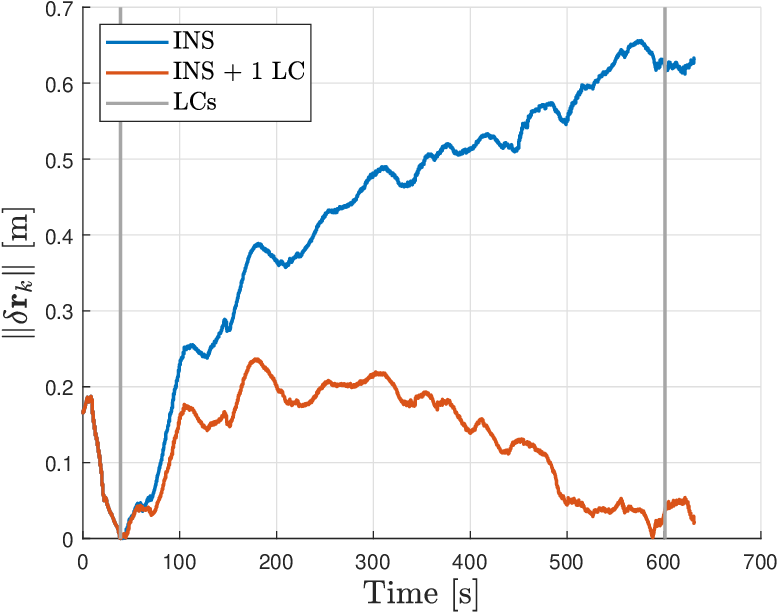}
  \caption{
    Relative displacement error $\| \delta \mbfhat{r}_k \|$ incorporating only the \textit{last} LC measurement into the field data using the proposed pipeline.
    Note how the \ac{LC} correction propagates smoothly to other poses between the loop closure locations.
  }
  \label{fig:wiarton/kummerle_energy_LC7Only.eps}
\end{figure}

\begin{figure*}[pt]
  \centering
  \begin{multicols}{2}
    \begin{subfigure}{\columnwidth}
      \includegraphics[width=\textwidth]{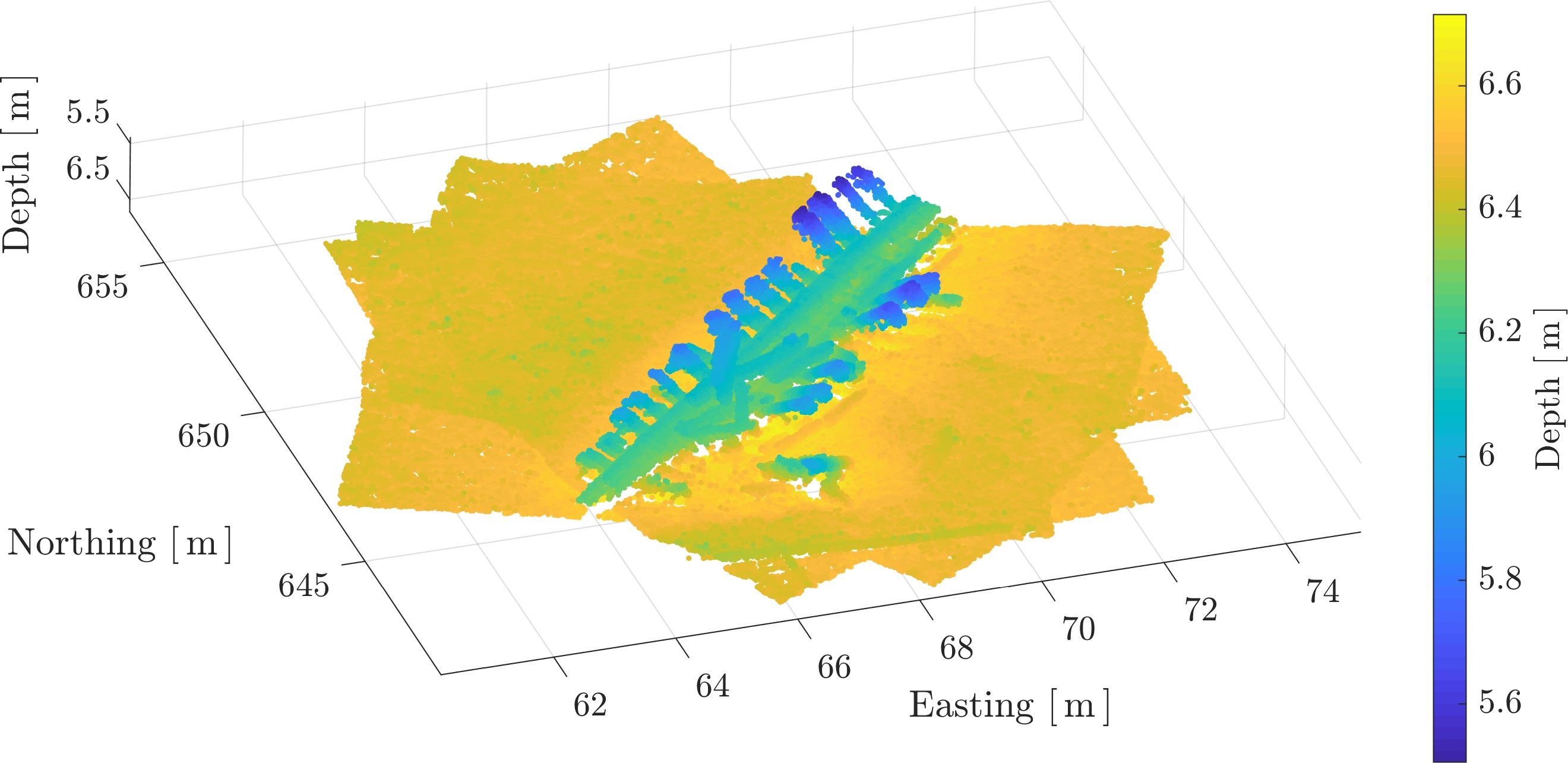}
      \caption{Ground-truth elevation (`INS~+~GPS')}
      \label{fig:elevation_gt}
    \end{subfigure}
    \begin{subfigure}{\columnwidth}
      \includegraphics[width=\textwidth]{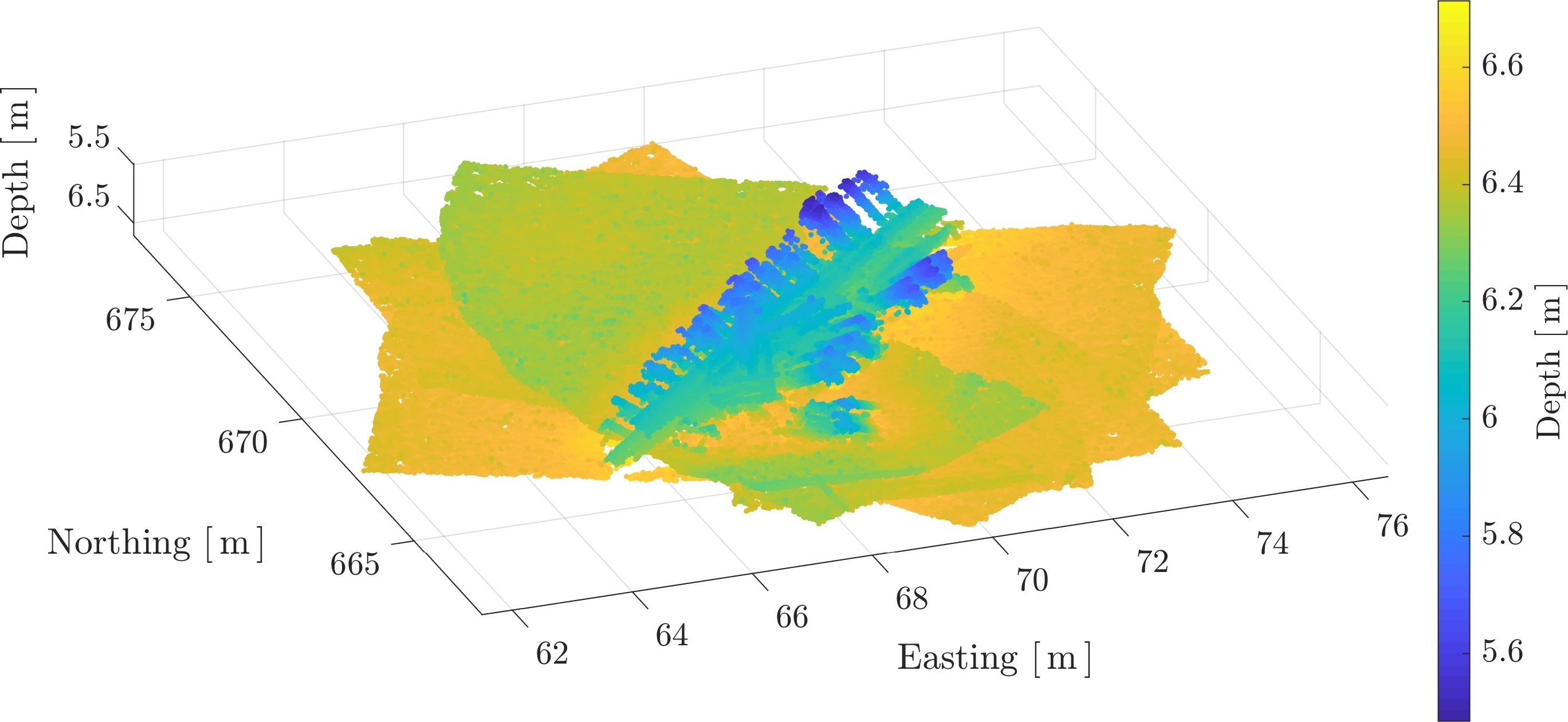}
      \caption{Prior elevation (`INS')}
      \label{fig:elevation_prior}
    \end{subfigure}
    \begin{subfigure}{\columnwidth}
      \includegraphics[width=\textwidth]{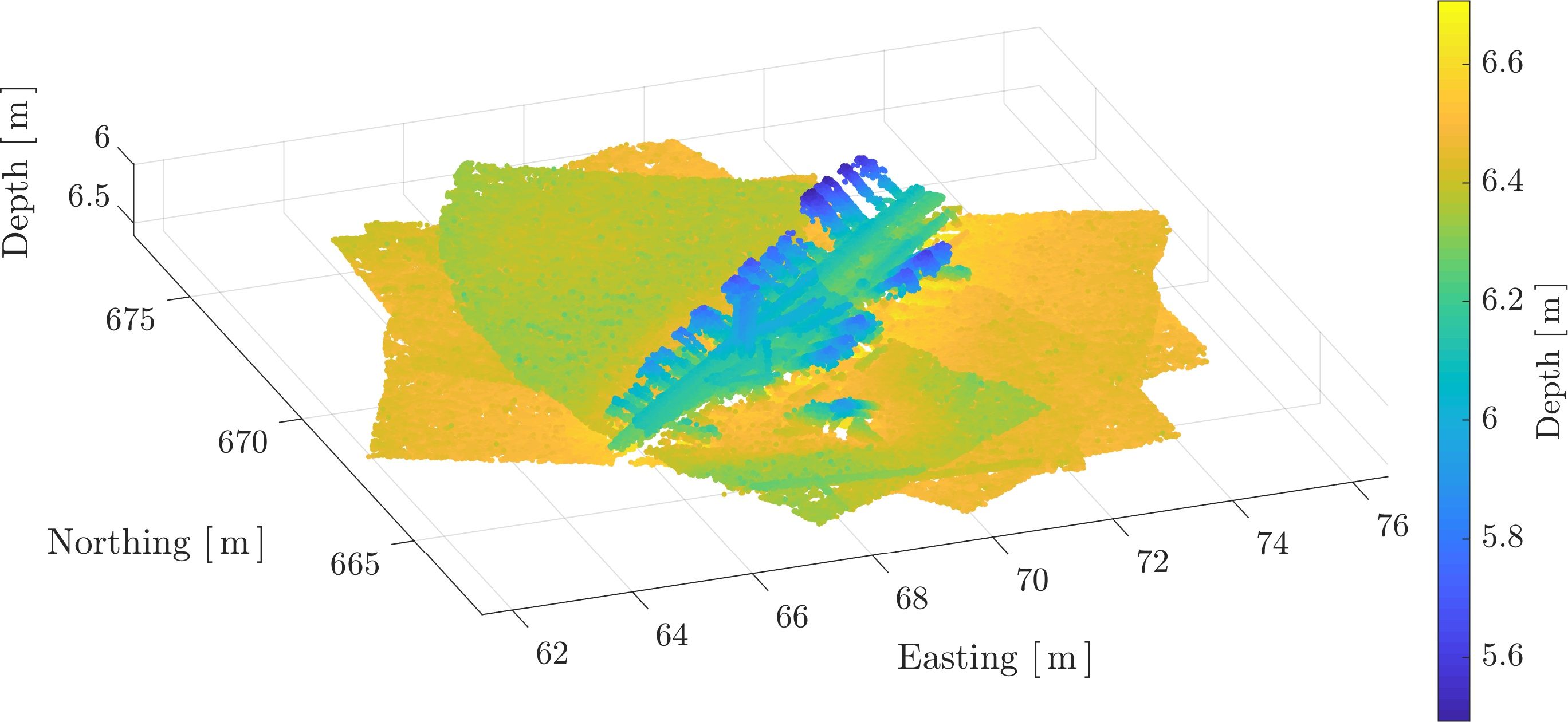}
      \caption{Posterior elevation (`INS~+~LC')}
      \label{fig:elevation_posterior}
    \end{subfigure}
    \begin{subfigure}{\columnwidth}
      \includegraphics[width=\textwidth]{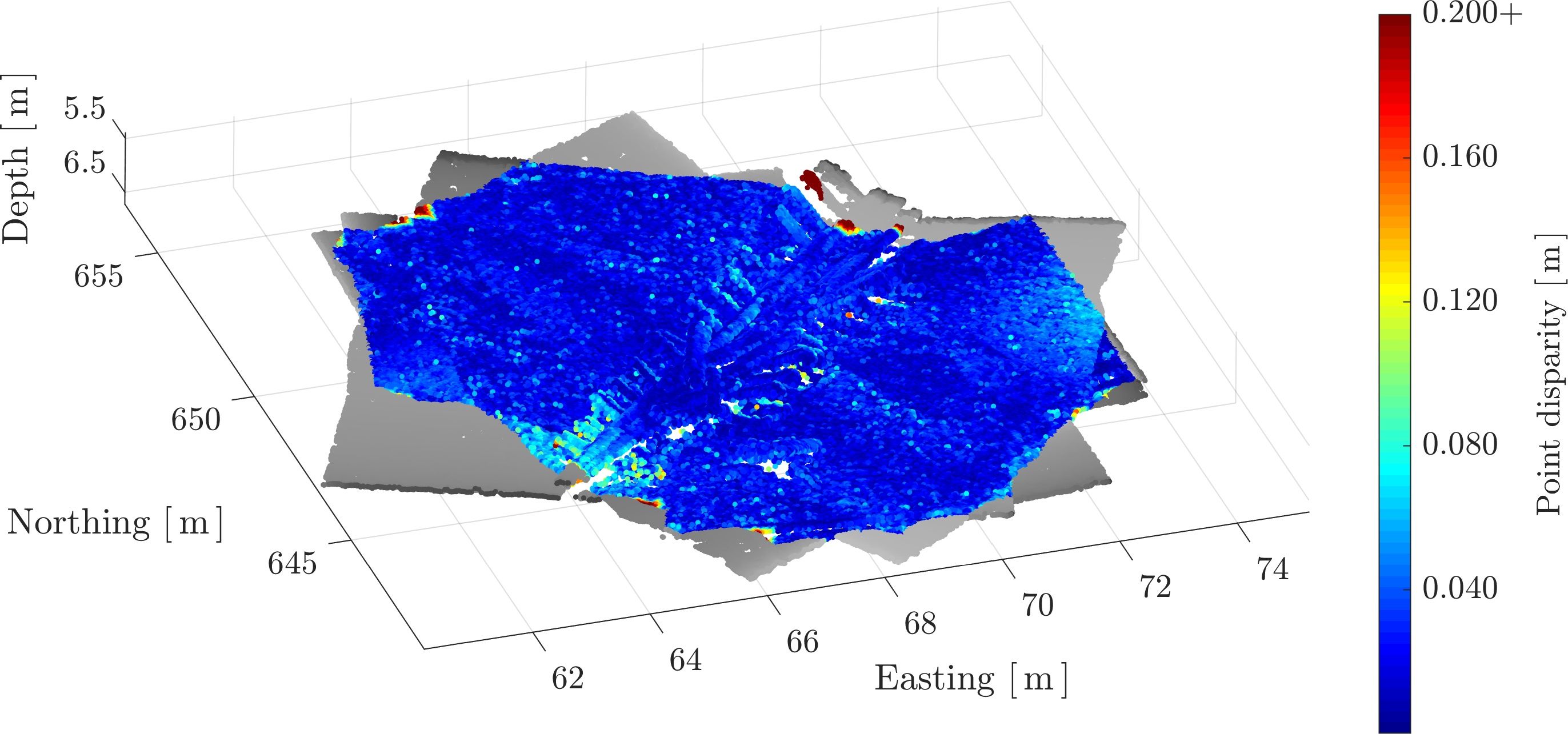}
      \caption{Ground-truth disparity (`INS~+~GPS')}
      \label{fig:disp_gt}
    \end{subfigure}
    \begin{subfigure}{\columnwidth}
      \includegraphics[width=\textwidth]{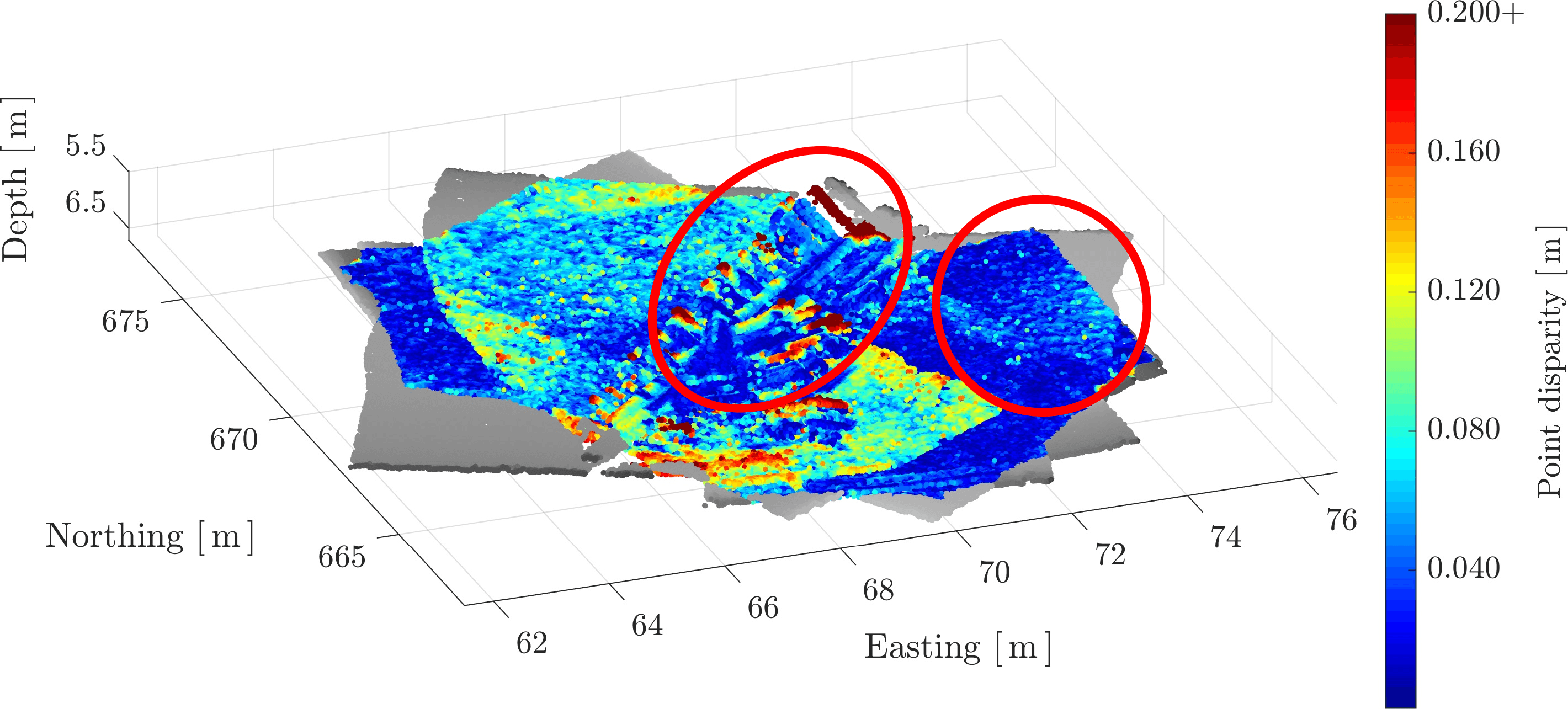}
      \caption{Prior disparity (`INS')}
      \label{fig:disp_prior}
    \end{subfigure}
    \begin{subfigure}{\columnwidth}
      \includegraphics[width=\textwidth]{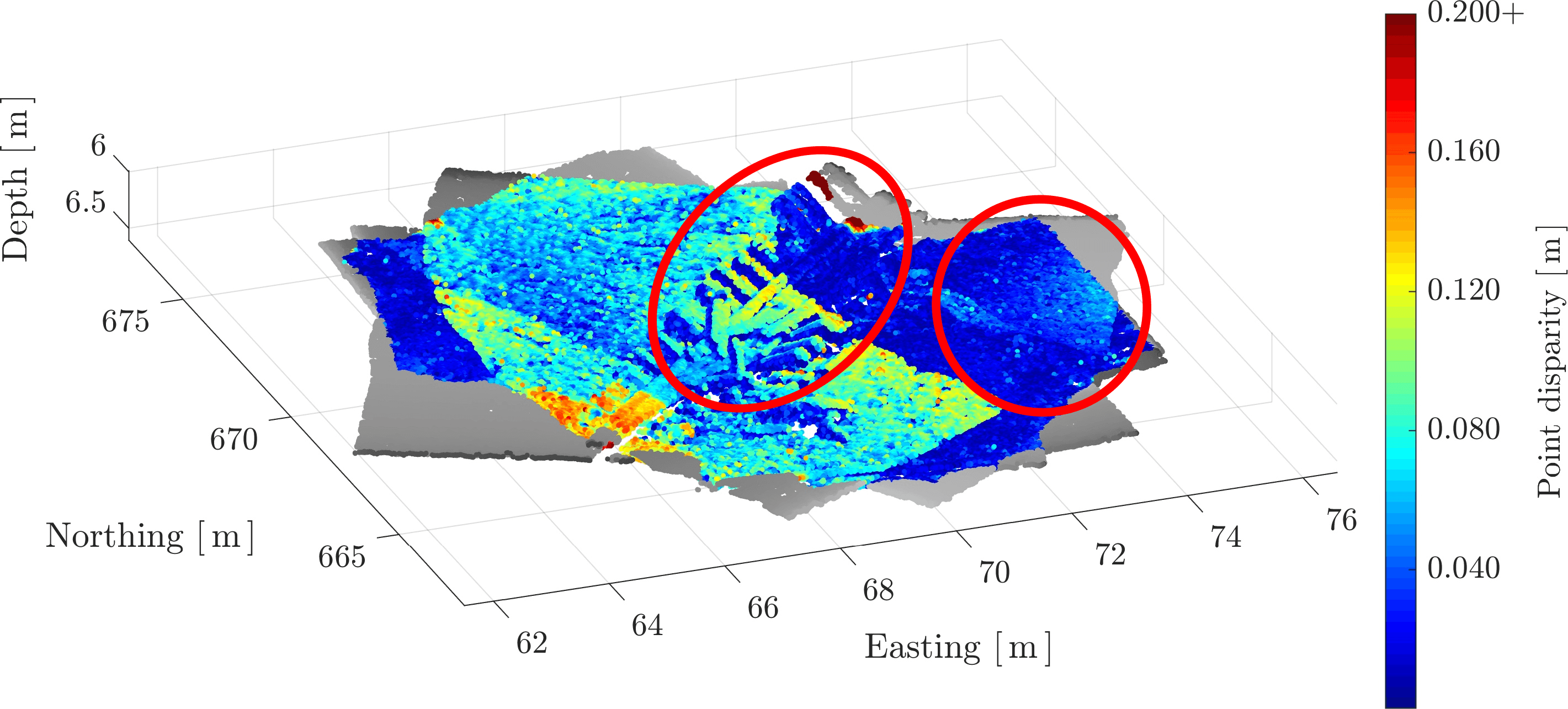}
      \caption{Posterior disparity (`INS~+~LC')}
      \label{fig:disp_posterior}
    \end{subfigure}
  \end{multicols}
  \caption{
    The shipwreck area before and after batch estimation. The left column shows point clouds coloured by depth, whereas the right column shows point clouds coloured by disparity error~\cite{Roman_Consistency_2006}.
    The posterior disparity map \Cref{fig:disp_posterior} has fewer red points (see red ellipses), indicating a reduction in disparity error.  Note the green patch in \Crefrange{fig:disp_prior}{fig:disp_posterior} is due to a bias in the `INS' depth estimate $\tilde{z}$ on the first pass over the wreck, which cannot be corrected by the proposed pipeline in current form.
  }
  \label{fig:Visualizing point clouds wiarton}
\end{figure*}

\begin{figure*}[tb]
  \centering
  \begin{subfigure}{0.42\textwidth}
    \includegraphics[width=\textwidth]{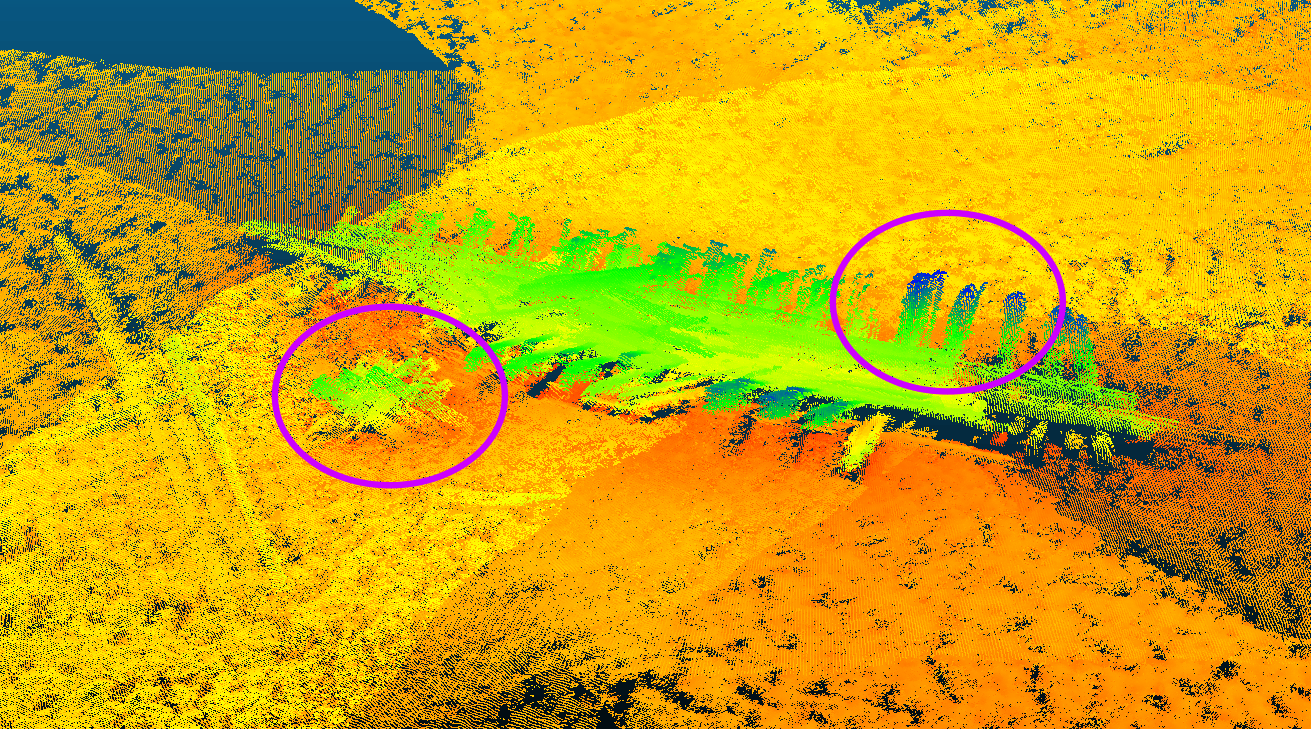}
    \caption{Prior elevation (`INS').}
    \label{fig:cc_elevation_prior}
  \end{subfigure}
  \vspace{6pt}
    \quad
\begin{subfigure}{0.42\textwidth}
    \includegraphics[width=\textwidth]{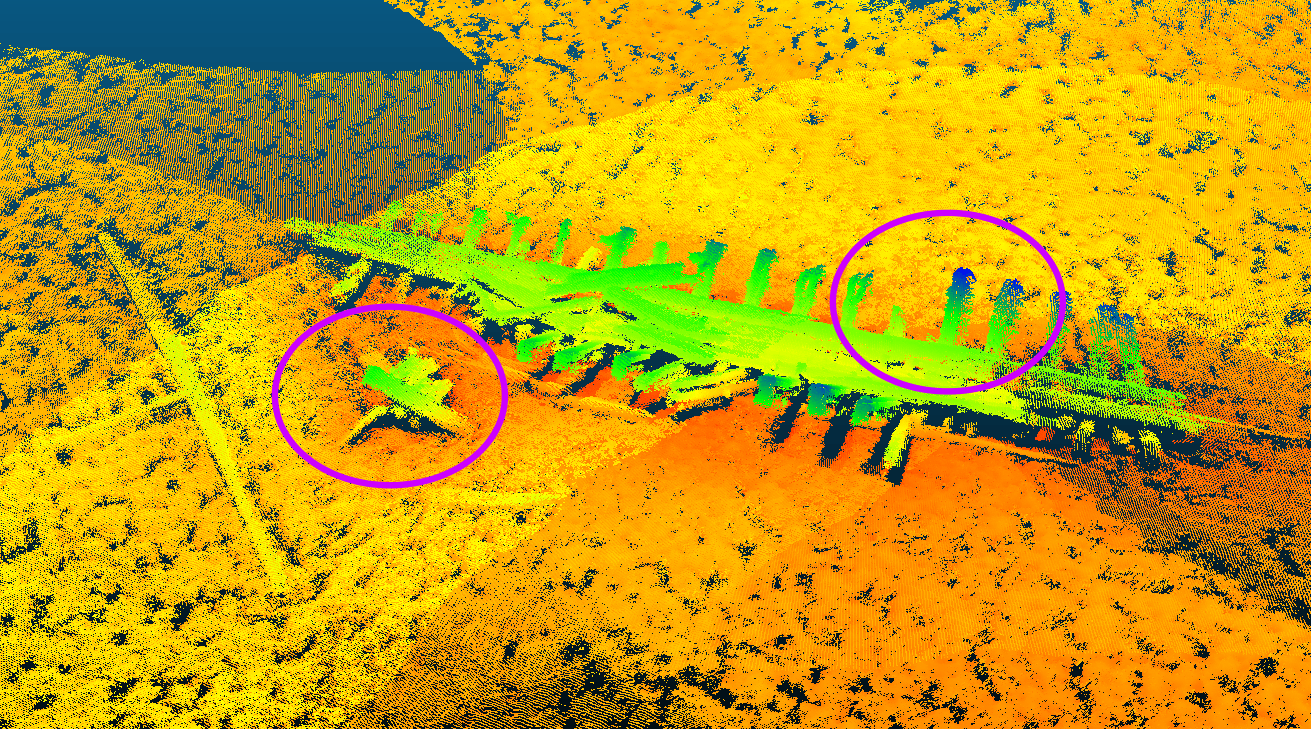}
    \caption{Posterior elevation (`INS~+~LC').}
    \label{fig:cc_elevation_posterior}
  \end{subfigure}

  \begin{subfigure}{0.42\textwidth}
    \includegraphics[width=\textwidth]{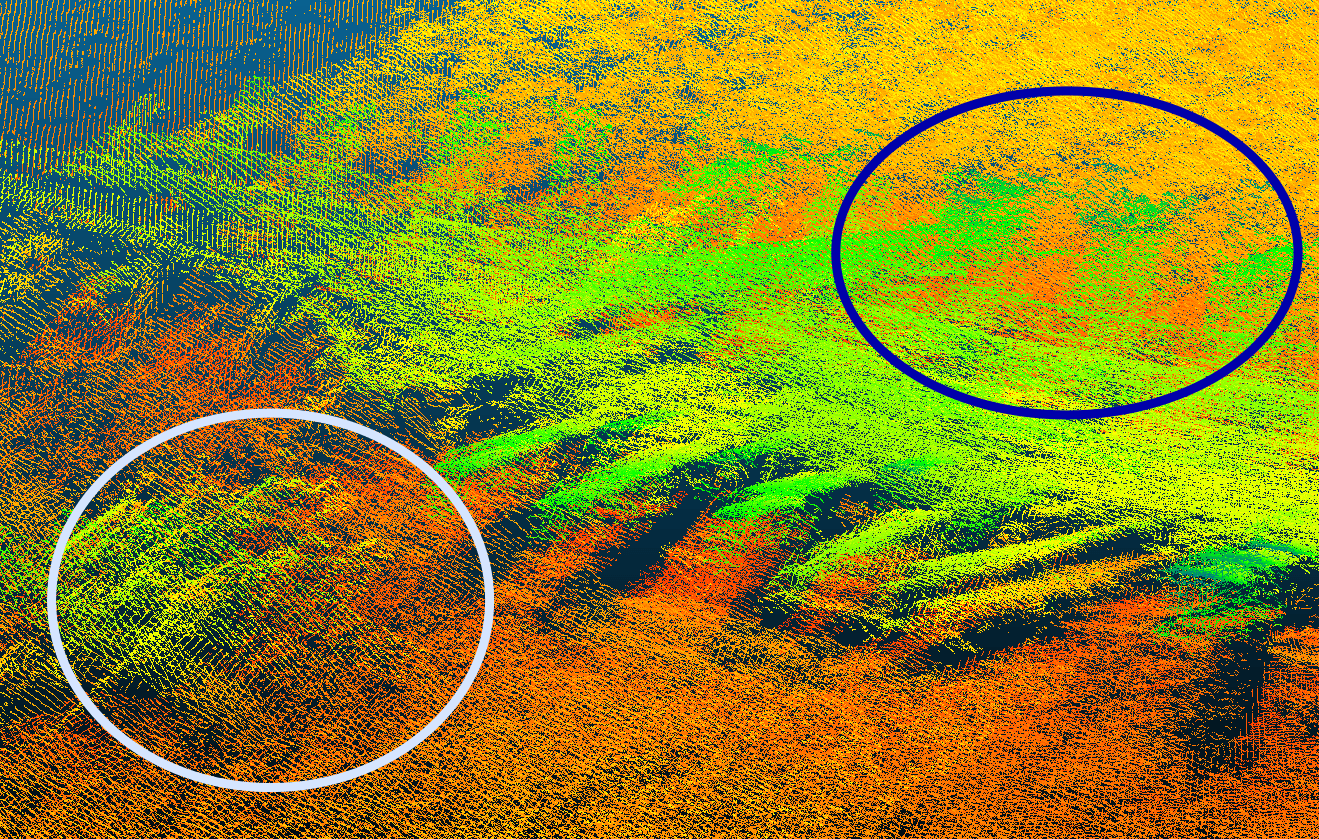}
    \caption{Prior elevation (`INS'), zoom.}
    \label{fig:cc_elevation_prior_zoom}
  \end{subfigure}
  \quad
  \begin{subfigure}{0.42\textwidth}
    \includegraphics[width=\textwidth]{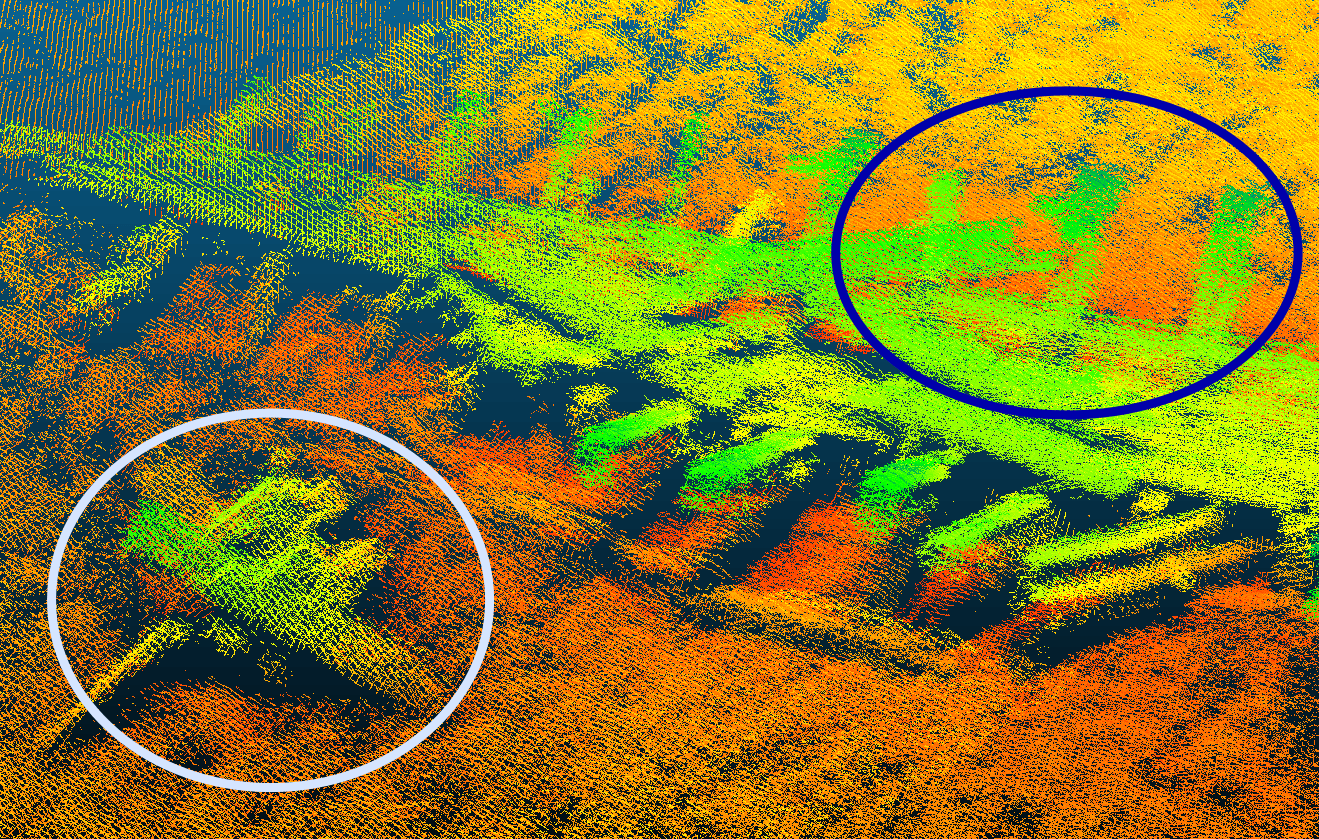}
    \caption{Posterior elevation (`INS~+~LC'), zoom.}
    \label{fig:cc_elevation_posterior_zoom}
  \end{subfigure}
  \caption{
    Images of the shipwreck area taken using CloudCompare~\cite{cloudcompare}, comparing the prior (`INS', left
    column) and posterior (`INS~+~LC', right column) navigation solutions.
    The bottom row shows a zoom of the shipwreck features.
  }
  \vspace{-3pt}
  \label{fig:shipwreck_passes}
\end{figure*}

\section{Conclusion}
\label{sec:Conclusion}
In this letter, the challenge of fusing measurements with processed state estimates in the absence of raw interoceptive measurements is addressed.
Specifically, loop-closure measurements computed using point-cloud scans from a Voyis~Insight~Pro underwater laser scanner are used to correct displacement estimates from a commercial DVL-INS.  The raw interoceptive measurements are not accessible, but are estimated from the {DVL-INS} output using convex optimization tools.  Estimated measurements are then used in a batch framework to smoothly propagate the effects of the LC corrections throughout the entire trajectory.


The simulated results demonstrate that the pipeline reduces a relative displacement error and that the posterior estimates computed via the method presented in \Cref{sec:Methodology} produce comparable results to the state estimates computed using corrupted ground-truth measurements.  Furthermore, the pipeline is also tested on experimental data collected during a field deployment.
The posterior trajectory generated using \ac{LC} measurements and the estimated interoceptive measurements showed a reduction in relative displacement error of more than 30~times compared to the estimates from the DVL-INS.

Future work would focus on \update{tuning the confidence of the posterior state estimate and} extending this pipeline to 3D by working on $SE(3)$ poses in order to update the depth and attitude estimates.

\appendices

\section{Deriving the Kalman Filter Equations}
\label{apx:Deriving the Kalman Filter Equations}
Consider the linear process model
\begin{align}
  \mbfrv{x}_{k}
   & =  \mbf{A}\mbfrv{x}_{k-1} + \mbf{B} \mbf{u}_{k-1} + \mbfrv{w}_{k-1},
\end{align}
where ${\mbfrv{x}_{k-1}\sim\mc{N}(\mbfhat{x}_{k-1}, \mbfhat{P}_{k-1})}$ is the state estimate at the previous time step and $\mbfrv{w}_{k-1}\sim\mc{N}\left(\mbf{0}, \mbf{Q}_{k-1} \right)$ is the process noise.
Furthermore, let the linear measurement model be
\begin{align}
  \mbfrv{y}_{k}
   & = \mbf{H}_{k} \mbfrv{x}_{k} + \mbf{M}_{k}\mbfrv{n}_{k},
\end{align}
where $\mbfrv{n}_{k}\sim\mc{N}\left(\mbf{0}, \mbf{R}_{k} \right)$ is the measurement noise.
%
%
The \ac{MAP} estimate is given by
\begingroup
\allowdisplaybreaks
\begin{align}
  \mbfhat{x}_{k}
   & =
  \argmax_{\mbf{x}_{k}\in \rnums^{n} } \pdf{\mbf{x}_{k}\mid \mbfhat{x}_{k-1}, \mbf{u}_{k-1}, \mbf{y}_{k}}
  \\
   & =
  \argmax_{\mbf{x}_{k}\in \rnums^{n} }
  \pdf{\mbf{y}_{k} \mid \mbf{x}_{k}}\pdf{\mbf{x}_{k}\mid \mbfhat{x}_{k-1}, \mbf{u}_{k-1}}
  \\
  \label{eq:mbfhatx as MAP}
   & =
  \argmax_{\mbf{x}_{k}\in \rnums^{n} }
  \eta
  \exp
  \left(
  -\f{1}{2}
  \norm{
    \mbf{y}_{k} - \mbf{H}_{k}\mbf{x}_{k}
  }^{2}_{\mbf{R}_{k}\inv}
  \right)
  \nonumber                                \\
   & \mathrel{\phantom{=}} \negmedspace {}
  \cdot
  \exp
  \left(
  -\f{1}{2}
  \norm{
    \mbf{x}_{k} - \mbf{A}_{k}\mbfhat{x}_{k-1} - \mbf{B}\mbf{u}_{k-1}
  }^{2}_{
    \mbf{Q}_{k-1}\inv
  }
  \right),
\end{align}
\endgroup
where $\eta$ is a normalizing constant.

Taking the negative log of \eqref{eq:mbfhatx as MAP} results in the equivalent least-squares optimization problem
\begin{align}
  \label{eq:mbfxhatk as LS}
  \mbfhat{x}_{k}
   & = \argmin_{\left\{ \mbf{x}_{k}\in \rnums^{n} \right\}} \f{1}{2}\norm{\mbf{J}_{k}\mbf{x}_{k} - \mbf{b}_{k}}^2_{\mbs{\Sigma}_{k}\inv},
\end{align}
where
\begingroup
\allowdisplaybreaks
\begin{align}
  \label{eq:KF LS matrices: J}
  \mbf{J}_{k}      & = \bbm \eye                                                        \\\linsysC_{k} \ebm,\\
  \label{eq:KF LS matrices: b}
  \mbf{b}_{k}
                   & =
  \bbm
  \linsysA\mbfhat{x}_{k-1} + \linsysB\mbf{u}_{k-1}
  \\
  \mbf{y}_{k}
  \ebm,
  \\
  \label{eq:KF LS matrices: Sigma}
  \mbs{\Sigma}_{k} & = \bbm \linsysA\mbfhat{P}_{k-1}\linsysA^{\trans} + \mbf{Q}_{k-1} & \\ & \linsysM_{k}\mbf{R}_{k}\linsysM_{k}^{\trans}\ebm.
\end{align}
\endgroup
The information matrix on the posterior estimate $\mbfhat{x}_{k}$ is~\cite{Barfoot_State_2017}
\begin{align}
  \label{eq:mbfhatPk begin}
  \mbfhat{P}_{k}\inv
   & = \cov{\mbfhatrv{x}_{k}}\inv                                                                                                                                            \\
   & = \mbf{J}^{\trans}\mbs{\Sigma}\inv\mbf{J}                                                                                                                              \\
  \label{eq:mbfhatPk end}
   & = \left(\linsysA\mbfhat{P}_{k-1}\linsysA^{\trans} + \mbf{Q}_{k-1} \right)\inv + \linsysC_{k}^{\trans} (\linsysM_{k}\mbf{R}_{k}\linsysM_{k}^{\trans})\inv\linsysC_{k}.
\end{align}
The optimal estimate is then
\begin{align}
  \label{eq:xhatk (KF) LS solution}
  \mbfhat{x}_{k} & = \underbrace{\left(\mbf{J}^{\trans}\mbs{\Sigma}\inv\mbf{J} \right)\inv}_{\mbfhat{P}_{k}}\mbf{J}^{\trans}\mbs{\Sigma}\inv\mbf{z}_{k} \\
                 & = \mbfhat{P}_{k}\mbf{J}^{\trans}\mbs{\Sigma}\inv \bbm \linsysA\mbfhat{x}_{k-1} + \linsysB\mbf{u}_{k-1}                                  \\ \mbf{y}_{k} \ebm.
\end{align}

\section{Inverting Strict LMIs}
\begin{lemma}
  \label{lemma:invertible strict LMI}
  Let $\mbf{X}, \mbf{Y}\in\mbb{S}^{n}$ be two positive definite matrices (\ie, $\mbf{X}, \mbf{Y} > 0$).
  Then, the relation
  \begin{align}
    \label{eq:lemma inverting strict LMI: Xinv - Y > 0}
    \mbf{X}\inv - \mbf{Y} & > 0
  \end{align}
  holds if and only if
  \begin{align}
    \label{eq:lemma inverting strict LMI: X - Yinv < 0}
    \mbf{X} - \mbf{Y}\inv & < 0.
  \end{align}
\end{lemma}%
\vspace{-6pt}%
\begin{proof}
  First, \eqref{eq:lemma inverting strict LMI: Xinv - Y > 0} implies \eqref{eq:lemma inverting strict LMI: X - Yinv < 0} will be shown.
  Using the Sherman-Morrison-Woodbury identity
  ~\cite{ Barfoot_State_2017},
  \vspace{-3pt}
  \begin{align}
    \left(\mbf{X}\inv - \mbf{Y} \right)\inv
     & =
    -\mbf{Y}\inv - \mbf{Y}\inv\left(\mbf{X} - \mbf{Y}\inv \right)\inv\mbf{Y}\inv.
    \vspace{-3pt}
  \end{align}
  The positive definiteness of \eqref{eq:lemma inverting strict LMI: Xinv - Y > 0} (and its inverse) implies
  \vspace{-3pt}
  \begin{align}
    \label{eq:lemma inverting strict LMI: suff 1}
    -\mbf{Y}\inv - \mbf{Y}\inv\left(\mbf{X} - \mbf{Y}\inv \right)\inv\mbf{Y}\inv
     & > 0.
     \vspace{-3pt}
  \end{align}
  Pre- and post multiplying \eqref{eq:lemma inverting strict LMI: suff 1} by the invertible $\mbf{Y}$ results in
  \vspace{-3pt}
  \begin{align}
    -\mbf{Y} - \left(\mbf{X} - \mbf{Y}\inv \right)\inv
     & > 0,
     \vspace{-3pt}
  \end{align}
  which is rearranged to give
  \vspace{-3pt}
  \begin{align}
    \left(\mbf{X} - \mbf{Y}\inv \right)\inv
     & < - \mbf{Y}
    < 0,
    \vspace{-9pt}
  \end{align}
  which in turn implies
  \vspace{-6pt}
  \begin{align}
    \mbf{X} - \mbf{Y}\inv
     & < 0.
     \vspace{-6pt}
  \end{align}%
  Second, \eqref{eq:lemma inverting strict LMI: X - Yinv < 0} implies \eqref{eq:lemma inverting strict LMI: Xinv - Y > 0} will be shown.
  Using the same Sherman-Morrison-Woodbury identity, the left-hand side of \eqref{eq:lemma inverting strict LMI: X - Yinv < 0} is written as
  \vspace{-3pt}
  \begin{align}
    \left(\mbf{X} - \mbf{Y}\inv \right)\inv
     & =
    \mbf{X}\inv - \mbf{X}\inv\left(\mbf{X}\inv - \mbf{Y} \right)\inv\mbf{X}\inv.
  \end{align}
  The positive definiteness of \eqref{eq:lemma inverting strict LMI: X - Yinv < 0} (and its inverse) implies
  \vspace{-3pt}
  \begin{align}
    \mbf{X}\inv - \mbf{X}\inv\left(\mbf{X}\inv - \mbf{Y} \right)\inv\mbf{X}\inv
     & > 0.
  \end{align}
  Pre- and post multiplying by the invertible matrix $\mbf{X}$ results in
  \vspace{-3pt}
  \begin{align}
    \mbf{X} - \left(\mbf{X}\inv - \mbf{Y} \right)\inv
     & < 0,
  \end{align}
  which is rearranged to give
  \vspace{-3pt}
  \begin{align}
    \left(\mbf{X}\inv - \mbf{Y} \right)\inv
     & > \mbf{X}
    > 0,
  \end{align}
  which in turn implies
  \vspace{-3pt}
  \begin{align}
    \mbf{X}\inv - \mbf{Y}
     & > 0.
  \end{align}
\end{proof}

\section{Existence and Nonuniqueness of the \acs{CSP}}
\label{sec:existence and nonuniqueness of CSP}
\begin{theorem}
  \label{thm:existence and nonuniqueness of CSP}
  Given positive definite matrices $\Pins_{k}, \Pins_{k-1}\in\mbb{S}^{n}$ (\ie, $\Pins_{k}, \Pins_{k-1}>0$) and a full rank matrix $\linsysA\in\rnums^{n\times n}$, there exists $\mbf{Q} > 0$ and $\mbs{\Omega}\geq 0$ such that
  \begin{align}
    \label{eq:existence and nonuniqueness of CSP: cov. constraint}
    \Pins_{k}\inv
     & =
    \left(
    \linsysA \Pins_{k-1} \linsysA^{\trans}
    + \mbf{Q}
    \right)\inv
    +
    \mbs{\Omega}
  \end{align}
  holds, and the solution is not unique.
\end{theorem}

\begin{proof}
  The existence of a solution is proved by showing an example that will always produce a valid solution.

  Set
  \begin{align}
    \mbf{Q}
     & \coloneqq \Pins_{k}
    > 0.
  \end{align}
  Then, the corresponding $\mbs{\Omega}$ that satisfies \eqref{eq:existence and nonuniqueness of CSP: cov. constraint} is given by
  \begin{align}
    \mbs{\Omega}
     & =
    \Pins_{k}\inv
    -
    \left(
    \linsysA \Pins_{k-1} \linsysA^{\trans}
    + \mbf{Q}
    \right)\inv
    \\
     & =
    \Pins_{k}\inv
    -
    \left(
    \linsysA \Pins_{k-1} \linsysA^{\trans}
    + \Pins_{k}
    \right)\inv
    \\
    \label{eq:existence of a solution: SMW}
     & =
    \Pins_{k}
    +
    \underbrace{
      \Pins_{k}
    }_{>0}
    \underbrace{
      \left(
      \linsysA \Pins_{k-1} \linsysA^{\trans}
      \right)\inv
    }_{>0}
    \underbrace{
      \Pins_{k}
    }_{>0}
    > 0,
  \end{align}
  where the Sherman-Morrison-Woodbury identity~\cite{Barfoot_State_2017} is used in \eqref{eq:existence of a solution: SMW}.
  To prove sufficiency, let $\mbs{\Omega}^{(1)}\geq 0$ and $\mbf{Q}^{(1)}$ satisfy \eqref{eq:existence and nonuniqueness of CSP: cov. constraint}.
  Furthermore, let $\mbfdel{Q}>0$ be any positive definite matrix.
  Then, set the new positive definite matrix to be
  \begin{align}
    \label{eq:existence and nonuniqueness of csp: Q2 def}
    \mbf{Q}^{(2)}
     & \coloneqq \mbf{Q}^{(1)} + \mbfdel{Q}
    > 0.
  \end{align}
  Letting
  \begin{align}
    \label{eq:existence and nonuniqueness of CSP: define D}
    \mbf{D}
     & \coloneqq
    \left(
    \linsysA \Pins_{k-1} \linsysA^{\trans}
    \right)\inv
    > 0
  \end{align}
  and inserting \eqref{eq:existence and nonuniqueness of csp: Q2 def} and \eqref{eq:existence and nonuniqueness of CSP: define D} to \eqref{eq:existence and nonuniqueness of CSP: cov. constraint} results in
  \begin{align}
    \label{eq:existence and nonuniqueness of CSP: cov. constraint with mbfdelQ}
    \mbs{\Omega}^{(2)}
     & =
    \Pins_{k}\inv
    -
    \left(
    \linsysA \Pins_{k-1} \linsysA^{\trans}
    + \mbf{Q}^{(1)} + \mbfdel{Q}
    \right)\inv
    \\
     & =
    \Pins_{k}\inv
    -
    \left(
    \mbf{D} + \mbfdel{Q}
    \right)\inv
    \\
    \label{eq:nonuniqueness of CSP: SMW}
     & =
    \Pins_{k}
    -
    \left(
    \linsysA \Pins_{k-1} \linsysA^{\trans}
    + \mbf{Q}^{(1)}
    \right)\inv
    \nonumber                                \\
     & \mathrel{\phantom{=}} \negmedspace {}
    +
    \underbrace{\mbf{D}\inv}_{>0}
    \underbrace{
      \left(\mbfdel{Q}\inv + \mbf{D}\inv \right)
    }_{>0}
    \underbrace{\mbf{D}\inv}_{>0}
    \\
     & >
    \Pins_{k}\inv
    -
    \left(
    \linsysA \Pins_{k-1} \linsysA^{\trans}
    + \mbf{Q}^{(1)}
    \right)\inv
    \\
     & = \mbs{\Omega}^{(1)},
  \end{align}
  where the Sherman-Morrison-Woodbury identity \cite[Sec.~2.2.7]{Barfoot_State_2017} is used in \eqref{eq:nonuniqueness of CSP: SMW}.
  Therefore, for every $\mbfdel{Q}>0$, there exists $\mbf{Q}^{(2)} > 0$ and $\mbs{\Omega}^{(2)}\geq 0$ that satisfy \eqref{eq:existence and nonuniqueness of CSP: cov. constraint} where $\mbf{Q}^{(1)}\neq \mbf{Q}^{(1)}$ and $\mbs{\Omega}^{(2)}\neq \mbs{\Omega}^{(1)}$.
  Therefore, there are infinitely many possible solutions.
\end{proof}

\section*{Acknowledgment}

The authors would like to thank Ryan~Wicks from Voyis for the experimental data valuable input, and Martin~J{\o}rgensen and Mathew~Grove of Sonardyne for access to simulation data and valuable advice.

\ifCLASSOPTIONcaptionsoff
  \newpage
\fi




\printbibliography

\end{document}